\documentclass{article}
\usepackage{log_2024}						

\usepackage{booktabs}						
\usepackage{multirow}						
\usepackage{amsfonts}						
\usepackage{graphicx}						
\usepackage{duckuments}						

\usepackage[numbers,compress,sort]{natbib}	

\usepackage{units}
\usepackage{subcaption}
\usepackage{amsthm}
\usepackage{wrapfig}
\usepackage{tikz}
\usepackage{enumitem}
\usepackage{changepage}

\theoremstyle{plain}
\newtheorem{theorem}{Theorem}

\newtheorem{lemma}{Lemma}

\theoremstyle{definition}

\newtheorem{assumption}{Assumption}
\theoremstyle{remark}

\theoremstyle{observation}
\newtheorem{observation}{Observation}

\allowdisplaybreaks

\newcommand{\oms}{\{\!\!\{}
\newcommand{\cms}{\}\!\!\}}

\newcommand{\eg}{e.\,g., }
\newcommand{\ie}{i.\,e., }

\title[What Do GNNs Actually Learn? Towards Understanding their Representations]{What Do GNNs Actually Learn? Towards Understanding their Representations}

\author[G. Nikolentzos et al.]{%
Giannis Nikolentzos\\
University of Peloponnese, Greece\\
\email{nikolentzos@uop.gr}\And
Michail Chatzianastasis\\
LIX, \'Ecole Polytecnique, IP Paris, France\\
\email{mixalisx97@gmail.com}\And
Michalis Vazirgiannis\\
LIX, \'Ecole Polytecnique, IP Paris, France\\
\email{mvazirg@lix.polytechnique.fr}
}

\begin{document}

\maketitle

\begin{abstract}
In recent years, graph neural networks (GNNs) have achieved great success in the field of graph representation learning.
Although prior work has shed light on the expressiveness of those models (\ie whether they can distinguish pairs of non-isomorphic graphs), it is still not clear what structural information is encoded into the node representations that are learned by those models.
In this paper, we address this gap by studying the node representations learned by four standard GNN models.
We find that some models produce identical representations for all nodes, while the representations learned by other models are linked to some notion of walks of specific length that start from the nodes.
We establish Lipschitz bounds for these models with respect to the number of (normalized) walks.
Additionally, we investigate the influence of node features on the learned representations.
We find that if the initial representations of all nodes point in the same direction, the representations learned at the $k$-th layer of the models are also related to the initial features of nodes that can be reached in exactly $k$ steps.
We also apply our findings to understand the phenomenon of oversquashing that occurs in GNNs.
Our theoretical analysis is validated through experiments on synthetic and real-world datasets.
\end{abstract}

\section{Introduction}
Graphs arise naturally in a wide variety of domains such as in bio- and chemo-informatics~\cite{stokes2020deep}, in social network analysis~\cite{easley2010networks} and in information sciences~\cite{hogan2021knowledge}.
There is thus a need for machine learning algorithms that can operate on graph-structured data, \ie algorithms that can exploit both the information encoded in the graph structure but also the information contained in the node and edge features.
Recently, graph neural networks (GNNs) emerged as a very promising method for learning on graphs, and have driven the rapid progress in the field of graph representation learning~\cite{wu2020comprehensive}.
 
Even though different types of GNNs were proposed in the past years, message passing models undoubtedly seem like a natural approach to the problem. 
These models, known as message passing neural networks (MPNNs)~\cite{gilmer2017neural}, employ a message passing (or neighborhood aggregation) procedure where each node aggregates the representations of its neighbors along with its own representation to produce new updated representations.
For graph-related tasks, MPNNs usually apply some permutation invariant readout function to the node representations to produce a representation for the entire graph.
The family of MPNNs has been studied a lot in the past few years, and there are now available dozens of instances of this family of models.
A lot of work has focused on investigating the expressive power of those models.
It was recently shown that standard MPNNs are at most as powerful as the Weisfeiler-Leman algorithm in terms of distinguishing non-isomorphic graphs~\cite{xu2019how,morris2019weisfeiler}.

The recent success of GNNs put graph kernels, another approach for graph–based machine learning, into the shade.
Unlike GNNs, graph kernels generate representations (implicit or explicit) that typically capture some substructure of graphs~\cite{nikolentzos2021graph}.
Such substructures include random walks~\cite{kashima2003marginalized,gartner2003graph}, shortest paths~\cite{borgwardt2005shortest} and subgraphs~\cite{shervashidze2009efficient,kriege2012subgraph}.
Therefore, the properties and the graph representations produced by most graph kernels are fully-understood.
This is not however the case for MPNNs since, despite the great activity in the field, still little is known about the properties of graphs that are captured in the representations learned by those models.

In this paper, we fill this gap by studying the node representations learned by MPNNs.
We first investigate what structural properties of graphs are captured in the learned representations of standard models.
To study those representations, we capitalize on Lipschitz continuity, a standard tool for analyzing representations of neural network models and for assessing their robustness to perturbations~\cite{gouk2021regularisation,huang2021training}.
We show that when all nodes are annotated with the same features, both GAT~\cite{velivckovic2018graph} and DGCNN~\cite{zhang2018end} embed all nodes into the same vector.
Furthermore, we show that the representations that emerge at the $k$-th layer of GCN~\cite{kipf2017semi} and GIN~\cite{xu2019how} are related to some notion of walks of length $k$ over the input graph.
This suggests that MPNNs suffer from the following limitation: structurally dissimilar nodes can have similar (or even identical) representations at some layer $k$ where $k > 1$.
We also study the impact of node features on the learned representations.
We show that if the initial features of all nodes point in the same direction, the node representations at the $k$-th layer of GCN and GIN are all related to the initial features of the nodes that can be reached in exactly $k$ steps from the node.
We finally study the problem of oversquashing~\cite{alon2021bottleneck} from the lens of our theoretical findings.
We verify our theoretical analysis in experiments conducted on synthetic and real-world datasets.
Our code is available at \url{https://github.com/giannisnik/gnn-representations}.


\section{Related Work}
While GNNs have been around for decades~\citep{sperduti1997supervised,scarselli2009graph,micheli2009neural}, it is only in recent years that the scientific community became aware of their power and potential. 
The increased scientific activity in the field led to the development of a large number of models~\citep{bruna2014,li2015gated,duvenaud2015convolutional,atwood2016diffusion,defferrard2016convolutional}.
Those models were categorized into spectral and spatial approaches depending on which domain the convolutions (neighborhood aggregations) were performed.
Later, it was shown that all these models can be seen as instances of a single common framework~\citep{gilmer2017neural}.
These models, known as message passing neural networks (MPNNs), use a message passing scheme where nodes iteratively aggregate feature information from their neighbors.
Then, to compute a representation for the entire graph, MPNNs typically employ some permutation invariant readout function which aggregates the representations of all the nodes of the graph.
In the past few years, there have been proposed several extensions and improvements to the MPNN framework.
Most studies have focused on the message passing procedure and have proposed more expressive or permutation sensitive aggregation functions~\citep{murphy2019relational,seo2019discriminative,buterez2022graph,chatzianastasis2023graph}, schemes that incorporate different local structures or high-order neighborhoods~\citep{jin2020gralsp,abu2019mixhop}, non-Euclidean geometry approaches~\citep{chami2019hyperbolic}, while others have focused on efficiency~\citep{gallicchio2020fast}. 
Fewer works have focused on the pooling phase and have proposed more advanced strategies for learning hierarchical graph representations~\citep{ying2018hierarchical,gao2019graph}.
Note also that not all GNNs belong to the family of MPNNs~\citep{niepert2016learning,nikolentzos2020random,nikolentzos2023permute}, and that there exist models which process random walks sampled from the graph~\cite{tonshoff2023walking,wang2024non} and which can mitigate MPNNs' common symptoms such as oversmoothing and oversquashing.

A considerable amount of recent work has focused on characterizing the expressive power of GNNs.
Most of these studies compare GNNs against the WL algorithm and its variants~\citep{kiefer2020weisfeiler} to investigate what classes of non-isomorphic graphs they can distinguish.
For instance, it has been shown that standard GNNs are not more powerful than the $1$-WL algorithm~\citep{xu2019how,morris2019weisfeiler}.
Other studies capitalized on high-order variants of the WL algorithm to derive new models that are more powerful than standard MPNNs~\citep{morris2019weisfeiler,morris2020weisfeiler}.
Recent research has investigated the expressive power of $k$-order GNNs in terms of their ability to distinguish non-isomorphic graphs.
In particular, it has been shown that $k$-order GNNs are at least as powerful as the $k$-WL test in this regard~\citep{maron2019invariant}.
Various approaches have also been proposed to enhance the expressive power of GNNs beyond that of the WL test.
These include encoding vertex identifiers~\citep{vignac2020building}, incorporating all possible node permutations~\citep{murphy2019relational,dasoulas2020coloring}, using random features~\citep{sato2021random,abboud2020surprising}, utilizing node features~\citep{you2021identity}, incorporating spectral information~\citep{balcilar2021breaking}, utilizing simplicial and cellular complexes~\citep{bodnar2021weisfeiler,bodnar2021weisfeiler2} and directional information~\citep{beaini2021directional}.
It has also been shown that extracting and processing subgraphs can further enhance the expressive power of GNNs~\citep{nikolentzos2020k,zhang2021nested,bevilacqua2021equivariant}.
For instance, it has been suggested that expressive power of GNNs can be increased by aggregating the representations of subgraphs produced by standard GNNs, which arise from removing one or more vertices from a given graph~\citep{cotta2021reconstruction,papp2021dropgnn}.
The above studies mainly focus on whether GNNs can distinguish pairs of non-isomorphic graph.
However, it still remains unclear what kind of structural information is encoded into the node representations learned by GNNs.
Some recent works have proposed models that aim to learn representations that preserve some notion of distance of nodes~\citep{nikolentzos2023weisfeiler}, however, they do not shed light into the representations generated by standard models.
The work closest to ours is the one proposed by~\citet{chuang2022tree}, where the authors propose the Tree Mover's Distance, a pseudometric for node-attributed graphs, and study its relation to the generalization of GNNs.
Our work is also related to the work of~\citet{xu2018representation} where the authors use the concept of walks to define the effective range of nodes that any given node's representation draws from.
However, while this work studies the range of nodes whose features affect a given node's representation, we focus on the exact node representations that are learned by the model.
Finally,~\citet{yehudai2021local} capitalize on local computation trees and graph patterns similar to the ones studied in this paper to investigate the GNNs' ability to generalize to larger graphs.

\section{Preliminaries}

\subsection{Notation}
Let $\mathbb{N}$ denote the set of natural numbers, \ie $\{1,2,\ldots\}$. 
Then, $[n] = \{1,\ldots,n\} \subset \mathbb{N}$ for $n \geq 1$.
Let also $\oms \cms$ denote a multiset, \ie a generalized concept of a set that allows multiple instances for its elements.
Let $G = (V,E)$ be a (directed) graph, where $V$ is the vertex set and $E$ is the edge set.
We will denote by $n$ the number of vertices and by $m$ the number of edges, \ie $n = |V|$ and $m = |E|$.
The adjacency matrix $\mathbf{A} \in \mathbb{R}^{n \times n}$ encodes the edge information in a graph.
The element of the $i^{\text{th}}$ row and $j^{\text{th}}$ column is equal to $1$ if there is an edge between $v_i$ and $v_j$, and $0$ otherwise.
Let $\mathcal{N}(v)$ denote the the neighbourhood of vertex $v$, \ie the set $\{u \mid (u,v) \in E\}$.
The degree of a vertex $v$ is $d(v) = |\mathcal{N}(v)|$.
A walk in graph $G$ is a sequence of vertices $v_1,v_2,\ldots,v_{k+1}$ where $v_i \in V$ and $(v_i,v_{i+1}) \in E$ for $1 \leq i \leq k$.
We denote by $w_v^{(k)}$ the number of walks of length $k$ starting from node $v$.
Finally, let $\tilde{w}_v^{(k)}$ denote the sum of normalized walks of length $k$ where each walk $(v_1, v_2, \ldots, v_k)$ is normalized as follows $\nicefrac{1}{\big( (1+d(v_2)) \ldots (1+d(v_{k-1})) \sqrt{(1+d(v_1)) (1+d(v_k))} \big)}$.

\begin{table}[t]
    \centering
    \caption{Neighborhood aggregation schemes of the four considered models.}
    \tiny
    \begin{tabular}{|c|c|}
    \hline
    Model & Update Equation \\
    \hline
    \textbf{GCN} & $\mathbf{h}_v^{(k)} = \text{ReLU} \Bigg( \sum_{u \in \mathcal{N}(v) \cup \{v\}} \frac{\mathbf{W}^{(k)} \mathbf{h}_u^{(k-1)}}{\sqrt{(1+d(v)) (1+d(u))}} \Bigg)$ \\ \hline 
    \textbf{DGCNN} & $\mathbf{h}_v^{(k)} = f \Bigg( \sum_{u \in \mathcal{N}(v) \cup \{ v\}} \frac{1}{d(v)+1} \mathbf{W}^{(k)} \mathbf{h}_u^{(k-1)} \Bigg)$ \\
    \hline
    \textbf{GAT} & $\mathbf{h}_v^{(k)} = \sigma \Bigg( \sum_{u \in \mathcal{N}(v)} \alpha_{vu} \mathbf{W}^{(k)} \mathbf{h}_u^{(k-1)} \Bigg)$ \\ \hline
    \textbf{GIN-$\epsilon$} & $\mathbf{h}_v^{(k)} = \text{MLP}^{(k)} \Bigg( \Big( 1 + \epsilon^{(k)} \Big) \mathbf{h}_v^{(k-1)} + \sum_{u \in \mathcal{N}(v)} \mathbf{h}_u^{(k-1)} \Bigg)$ \\
    \hline
    \end{tabular}
    \label{tab:gnns}
\end{table}

\subsection{Message Passing Neural Networks}
As already discussed, most GNNs can be unified under the MPNN framework~\cite{gilmer2017neural}.
These models follow a neighborhood aggregation scheme, where each node representation is updated based on the aggregation of its neighbors' representations.
Let $\mathbf{h}_v^{(0)}$ denote node $v$'s initial feature vector.
Then, for a number $K$ of iterations, MPNNs update node representations as follows:
\begin{equation*}
    \begin{split}
        \mathbf{m}_v^{(k)} &= \text{AGGREGATE}^{(k)}  \Big( \oms \mathbf{h}_u^{(k-1)} \mid u \in \mathcal{N}(v) \cms \Big) \\
        \mathbf{h}_v^{(k)} &= \text{COMBINE}^{(k)}  \Big( \mathbf{h}_v^{(k-1)}, \mathbf{m}_v^{(k)}  \Big)
    \end{split}
    \label{eq:gnn_general}
\end{equation*}
where $\text{AGGREGATE}^{(k)}$ is a permutation invariant function.
By defining different $\text{AGGREGATE}^{(k)}$ and $\text{COMBINE}^{(k)}$ functions, we obtain different MPNN instances.
In this study, we consider the neighborhood aggregation schemes of four models, namely ($1$) Graph Convolution Network (GCN)~\cite{kipf2017semi}; ($2$)  Deep Graph Convolutional Neural Network (DGCNN)~\cite{zhang2018end}; ($3$) Graph Attention Network (GAT)~\cite{velivckovic2018graph}; and ($4$) Graph Isomorphism Network (GIN)~\cite{xu2019how}.
The aggregation schemes of the four models are illustrated in Table~\ref{tab:gnns}.
Note that for directed graphs, $\mathcal{N}(v)$ is a set that contains the in-neighbors of $v$.

For node-level tasks, the final node representations $\mathbf{h}_v^{(K)}$ can be directly passed to a fully-connected layer for prediction.
For graph-level tasks, a graph representation is obtained by aggregating the final representations of its nodes: $\mathbf{h}_G = \text{READOUT} \Big( \oms \mathbf{h}_v^{(K)} \mid v \in G \cms \Big)$.
The $\text{READOUT}$ function is typically a differentiable permutation invariant function such as the sum or mean operator.

\section{What Do MPNNs Actually Learn?}
We next investigate what structural properties of nodes these four considered models can capture.
Nodes are usually annotated with features that reveal information about their neighborhoods.
Such features include their degree or even more sophisticated features such as counts of certain substructures~\cite{bouritsas2022improving} or those extracted from Laplacian eigenvectors~\cite{dwivedi2020benchmarking}.
We are interested in identifying properties that are captured purely by these models.
Thus, we assume that no such features are available, and we annotate all nodes with the same feature vector or scalar.

\begin{theorem}
    Let $\mathcal{G}=\{G_1, \ldots,G_N\}$ be a collection of graphs.
    Let also $\mathcal{V}=V_1 \cup \ldots \cup V_N$ denote the set that contains the nodes of all graphs.
    All nodes are initially annotated with the same representation.
    Without loss of generality, we assume that they are annotated with a single feature equal to $1$, \ie $\mathbf{h}_v^{(0)} = 1$ $\forall v \in \mathcal{V}$.
    Then, after $k$ neighborhood aggregation layers:
    \begin{enumerate}[leftmargin=0.6cm]
        \item DGCNN and GAT both map all nodes to the same representation, \ie $\mathbf{h}_v^{(k)} = \mathbf{h}_u^{(k)}$ $\forall v,u \in \mathcal{V}$.
        \item GCN maps nodes to representations related to the sum of normalized walks of length $k$ starting from them:
        \begin{equation*}
            \Big| \Big| \mathbf{h}_v^{(k)} - \mathbf{h}_u^{(k)} \Big| \Big|_2 \leq \prod_{i=1}^k L_f^{(i)} \Big|\Big| \tilde{w}_v^{(k)} -  \tilde{w}_u^{(k)} \Big|\Big|_2 
        \end{equation*}
        where $L_f^{(i)}$ denotes the Lipschitz constant of the fully-connected layer of the $i$-th neighborhood aggregation layer.
        \item Under mild assumptions (biases of MLPs are ignored), GIN-$0$ maps nodes to representations that capture the number of walks of length $k$ starting from them:
        \begin{equation*}
            \Big| \Big| \mathbf{h}_v^{(k)} - \mathbf{h}_u^{(k)} \Big| \Big|_2 \leq \prod_{i=1}^k L_f^{(i)} \Big|\Big| w_v^{(k)} -  w_u^{(k)} \Big|\Big|_2 
        \end{equation*}
        where $L_f^{(i)}$ denotes the Lipschitz constant of the MLP of the $i$-th neighborhood aggregation layer.
    \end{enumerate}
    \label{thm:lipschitz}
\end{theorem}

The above result highlights the limitations of the considered models.
Specifically, our results imply that the DGCNN and GAT models encode no structural information of the graph into the learned node representations.
Furthermore, combined with a sum readout function, these representations give rise to graph representations that can only count the number of nodes of each graph.
If the readout function is the mean operator, then all graphs are embedded into the same vector.
With regards to the GIN-$0$ and GCN models, we have bounded their Lipschitz constants with respect to the number of walks and sum of normalized walks starting from the different nodes, respectively.

To experimentally verify the above theoretical results, we trained the GIN-$0$ and GCN models on the IMDB-BINARY and ENZYMES graph classification datasets.
For all pairs of nodes, we computed the Euclidean distance of the number of walks (resp. sum of normalized walks) of length $3$ starting from them.
We also computed the Euclidean distance of the representations of the nodes that emerge at the corresponding (\ie third) layer of GIN-$0$ (resp. GCN).
We finally computed the correlation of the two collections of Euclidean distances and the results are given in Figure~\ref{fig:corr_gnns}.
More details about the experimental protocol are provided in Appendix~\ref{sec:experimental_setup}.
Clearly, the results verify our theoretical results.
The distance of the number of walks is perfectly correlated with the distance of the representations generated by GIN-$0$ with no biases, while the distance of the sum of normalized walks is perfectly correlated with the distance of the representations produced by GCN.
We also computed the Euclidean distance of the representations of the nodes that emerge at the third layer of the standard GIN-$0$ model (with biases), and we compared them against the distances of the number of walks.
We can see that on both datasets, the emerging correlations are very high (equal to $0.99$).
We observed similar values of correlation on other datasets as well, which indicates that the magnitude of the bias terms of the MLPs is actually small and that our assumption of ignoring biases is by no means unrealistic.

\begin{figure*}[t]
    \centering
    \begin{subfigure}{0.32\textwidth}
    \includegraphics[width=\textwidth]{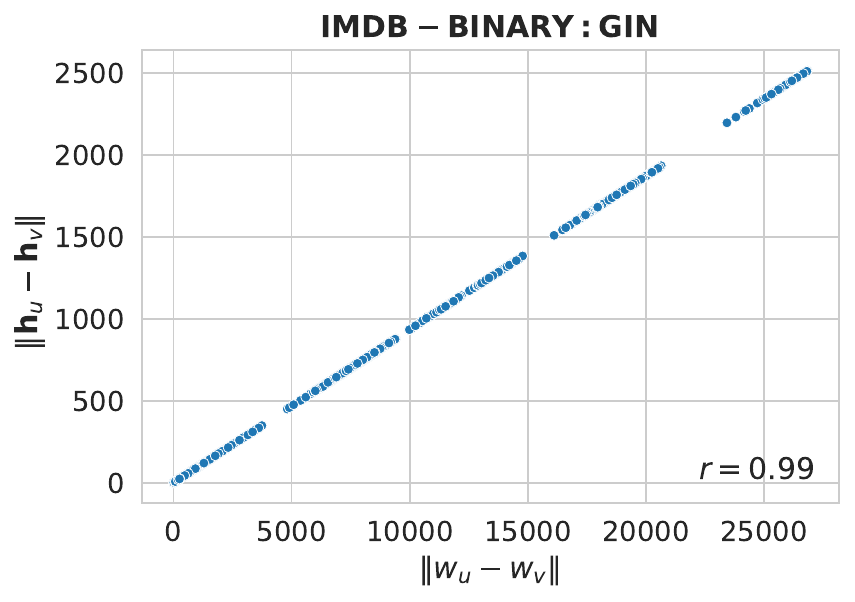}
    \end{subfigure}
    \begin{subfigure}{0.32\textwidth}
    \includegraphics[width=\textwidth]{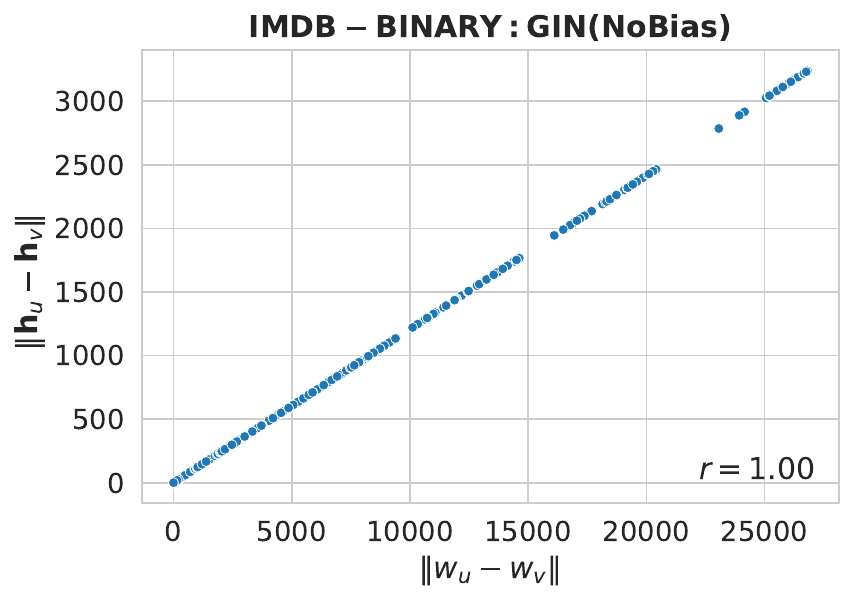}
    \end{subfigure}
    \begin{subfigure}{0.32\textwidth}
    \includegraphics[width=\textwidth]{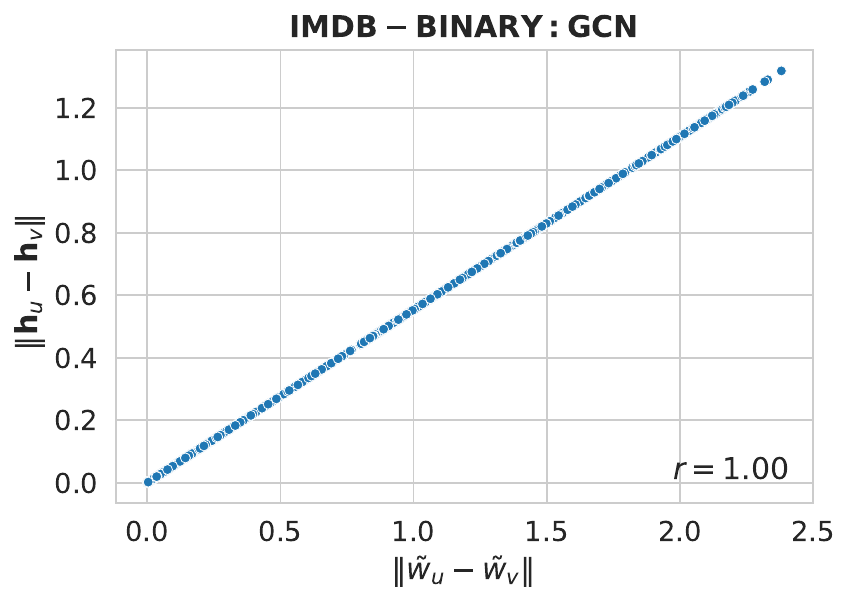}
    \end{subfigure}
    \begin{subfigure}{0.32\textwidth}
    \includegraphics[width=\textwidth]{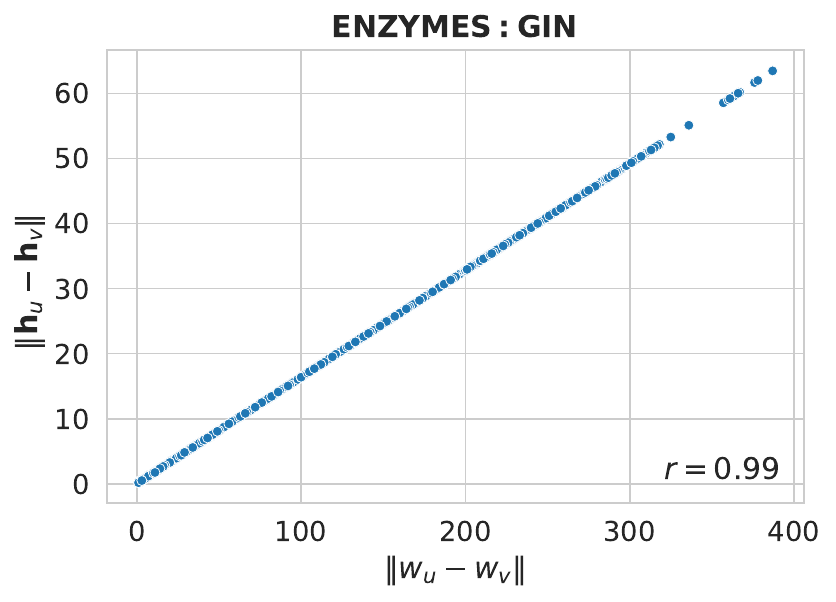}
    \end{subfigure}
    \begin{subfigure}{0.32\textwidth}
    \includegraphics[width=\textwidth]{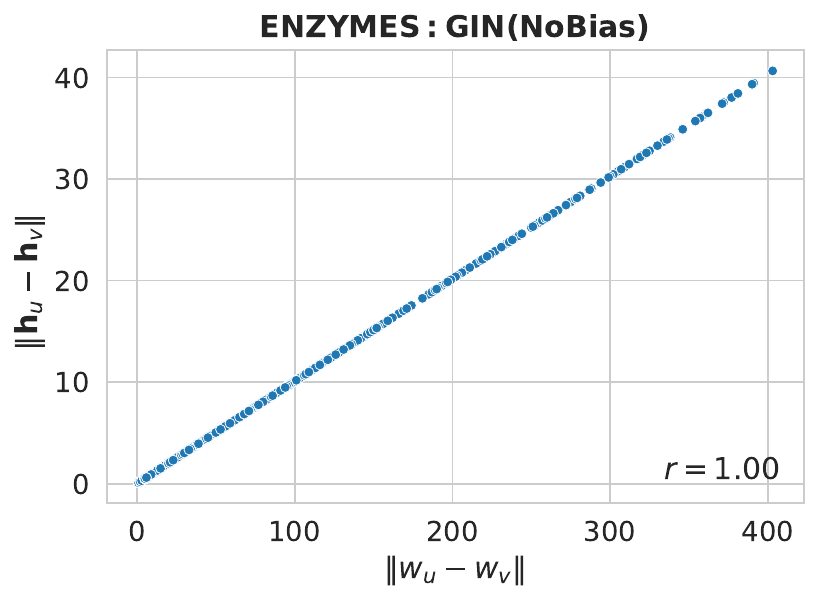}
    \end{subfigure}
    \begin{subfigure}{0.32\textwidth}
    \includegraphics[width=\textwidth]{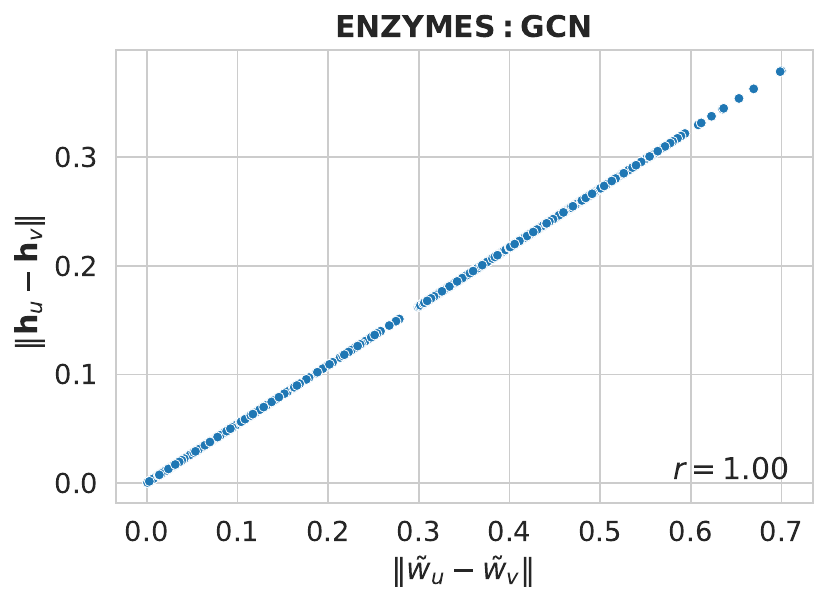}
    \end{subfigure}
    \caption{Euclidean distances of the representations generated at the third layer of the different models vs. Euclidean distances of the number of walks (or sum of normalized walks) of length $3$ starting from the different nodes.
    Nodes are initially annotated with a single feature equal to $1$.}
    \label{fig:corr_gnns}
\end{figure*}

\begin{figure}[t]
    \centering
    \begin{subfigure}{0.25\textwidth}
    \includegraphics[width=\textwidth]{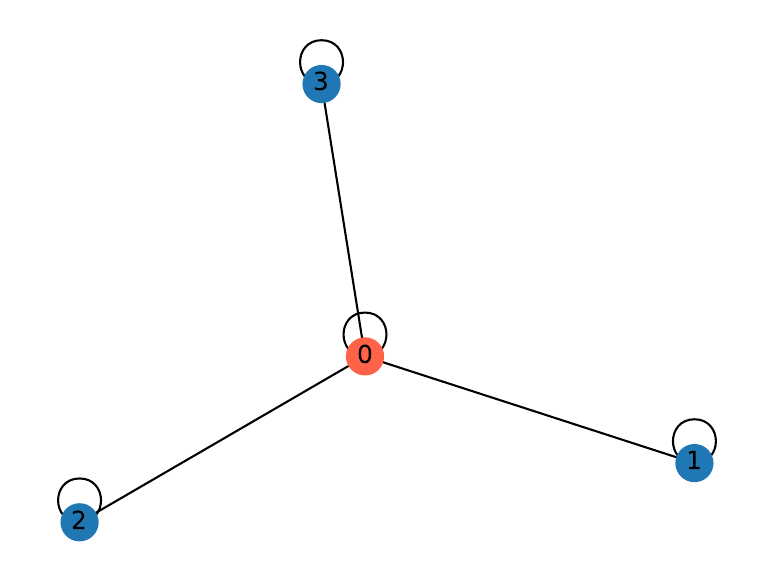}
    \end{subfigure}
    \begin{subfigure}{0.25\textwidth}
    \includegraphics[width=\textwidth]{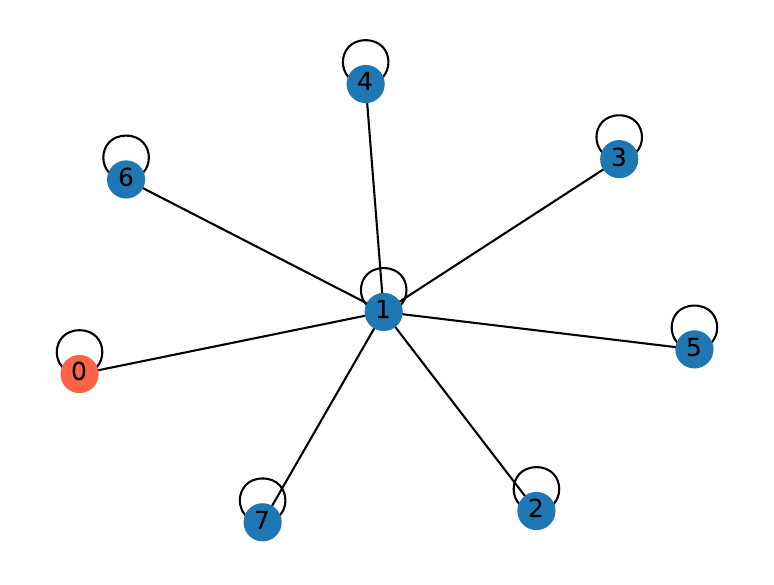}
    \end{subfigure}
    \begin{subfigure}{0.25\textwidth}
    \includegraphics[width=\textwidth]{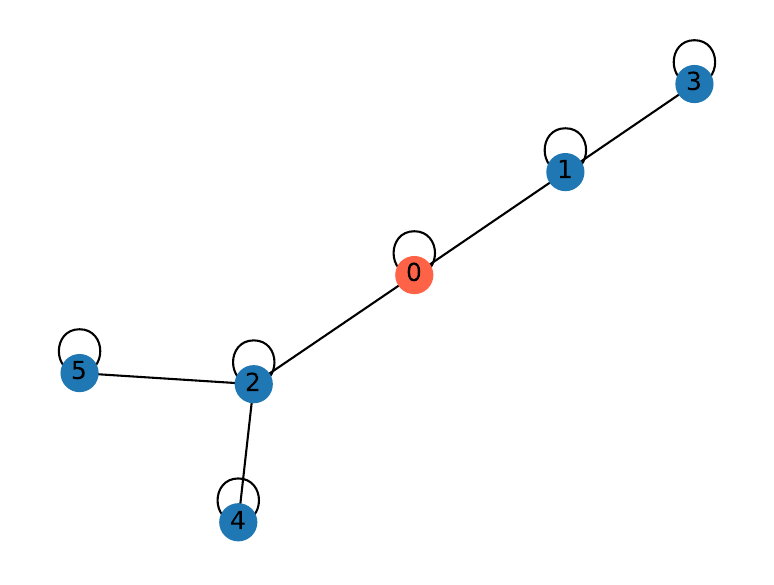}
    \end{subfigure}
    \caption{The number of walks of length $2$ starting from the red nodes of the three graphs is equal to $10$.
    A GIN model that consists of $2$ layers embeds these three nodes close to each other (or to the same representation in case there are no biases) even though they are structurally dissimilar.}
    \label{fig:examples_gin}
\end{figure}

\subsection{Can Structurally Dissimilar Nodes Obtain Identical Representations?}
Based on the above theoretical and empirical findings, it is clear that two nodes can have dissimilar representations at the $k$-th layer, but obtain similar representations at the $(k+1)$-th layer of some MPNN model.
For instance, for GCN, this can be the case if the two nodes have different sums of normalized walks of length $k$, but similar sums of normalized walks of length $k+1$.
Likewise, for GIN-$0$, this can occur if the two nodes have different numbers of walks of length $k$, but similar numbers of walks of length $k+1$.
But can two nodes that have different representations at the $k$-th layer of some MPNN model be embedded into the same vector at the $(k+1)$-th layer of the model?
\begin{observation}
    Let $\mathbf{h}_v^{(k)}$ denote node $v$'s representation at the $k$-th layer of the GCN or the GIN-$0$ model (biases of MLPs are ignored).
    There exist nodes $v$ and $u$ for which $|| \mathbf{h}_v^{(k)} - \mathbf{h}_u^{(k)} ||_2 > 0$, but $|| \mathbf{h}_v^{(k+1)} - \mathbf{h}_u^{(k+1)} ||_2 = 0$ no matter what are the values of the trainable parameters of the $(k+1)$-th layer of the model.
\end{observation}
Figure~\ref{fig:examples_gin} illustrates three nodes (the three red nodes) that have structurally dissimilar neighborhoods, but their representations produced by GIN-$0$ after two neighborhood aggregation layers are very similar to each other (or identical in case biases are omitted).
Let $v,u,z$ denote the three nodes.
For all three graphs, the number of walks of length $2$ starting from the red nodes is equal to $10$ (\ie $w_v^{(2)}=w_u^{(2)}=w_z^{(2)}=10$).
Note also that the three nodes have different values of degree (\ie number of walks of length $1$) from each other and thus GIN-$0$ could learn different representations for them after a single neighborhood aggregation layer.
We also provide in Figure~\ref{fig:examples_gcn} an example of two nodes (the two red nodes) that could obtain different representations at the first layer of a GCN model, but would obtain identical representations at the second layer of the model.
Let $v,u$ denote the two nodes.
For those two nodes, we have that $\tilde{w}_v^{(1)} \neq \tilde{w}_u^{(1)}$, but also that $\tilde{w}_v^{(2)} = \tilde{w}_u^{(2)} \approx 0.890$.

\begin{figure}[t]
    \centering
    \begin{subfigure}{0.25\textwidth}
    \includegraphics[width=\textwidth]{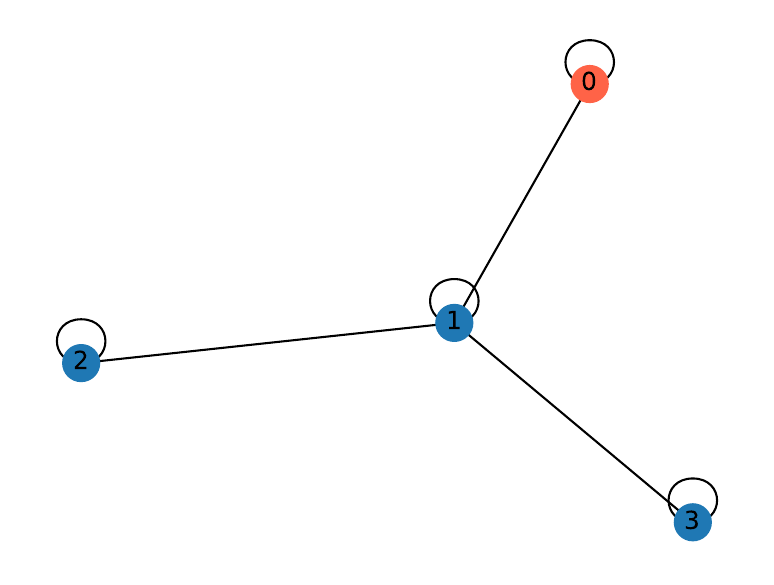}
    \end{subfigure}
    \hspace{2cm}
    \begin{subfigure}{0.25\textwidth}
    \includegraphics[width=\textwidth]{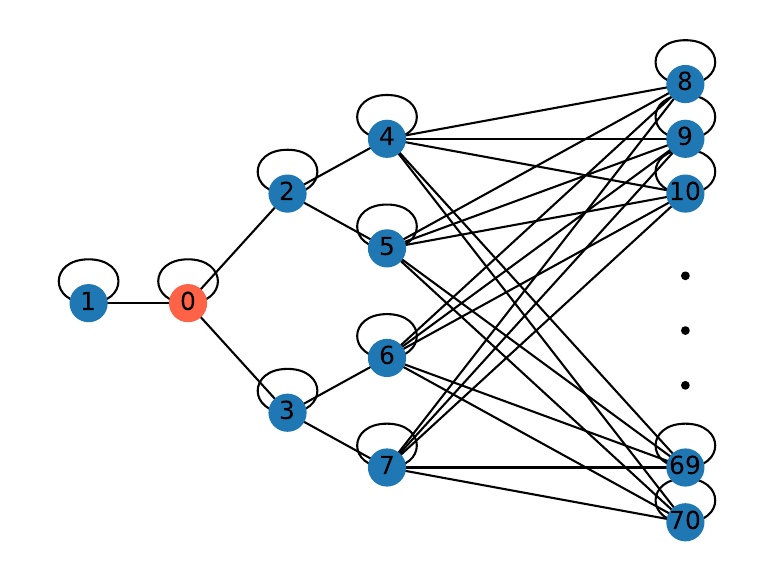}
    \end{subfigure}
    \caption{The sum of normalized walks of length $2$ starting from the red nodes of the two graphs is approximately equal to $0.890$.
    A GCN model that consists of $2$ layers embeds these two nodes to the same representation even though they are structurally dissimilar.}
    \label{fig:examples_gcn}
\end{figure}

In many settings, the local neighborhood of a node provides more information about its structural role than its more global neighborhood.
Indeed, two nodes that have very different values of degree from each other are undoubtedly structurally very dissimilar no matter how their $k$-hop (for $k>1$) neighborhoods look like. 
The jump connections proposed in~\cite{xu2018representation} allow a model to combine the representations of the different layers (\eg by concatenating them), thus capturing both local and more global information about each node's neighborhood.

\begin{figure*}[t]
    \centering
    \begin{subfigure}{0.32\textwidth}
    \includegraphics[width=\textwidth]{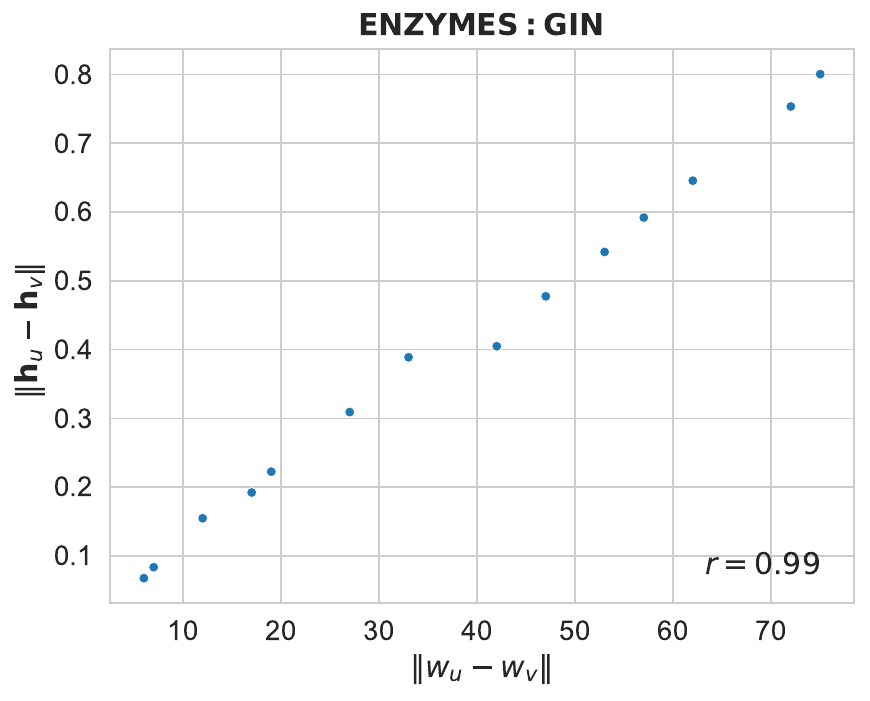}
    \end{subfigure}
    \begin{subfigure}{0.32\textwidth}
    \includegraphics[width=\textwidth]{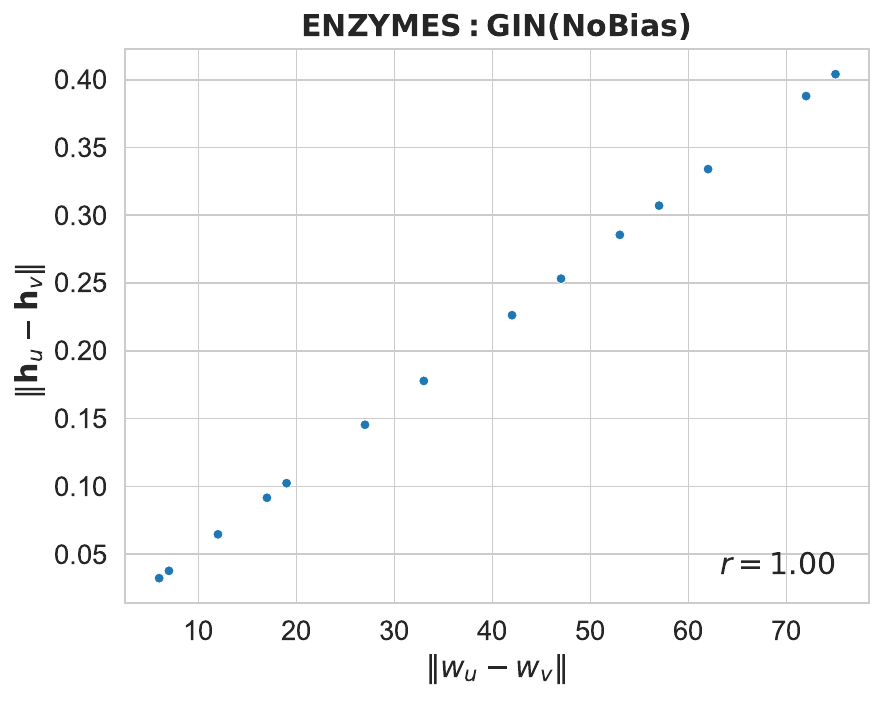}
    \end{subfigure}
    \begin{subfigure}{0.32\textwidth}
    \includegraphics[width=\textwidth]{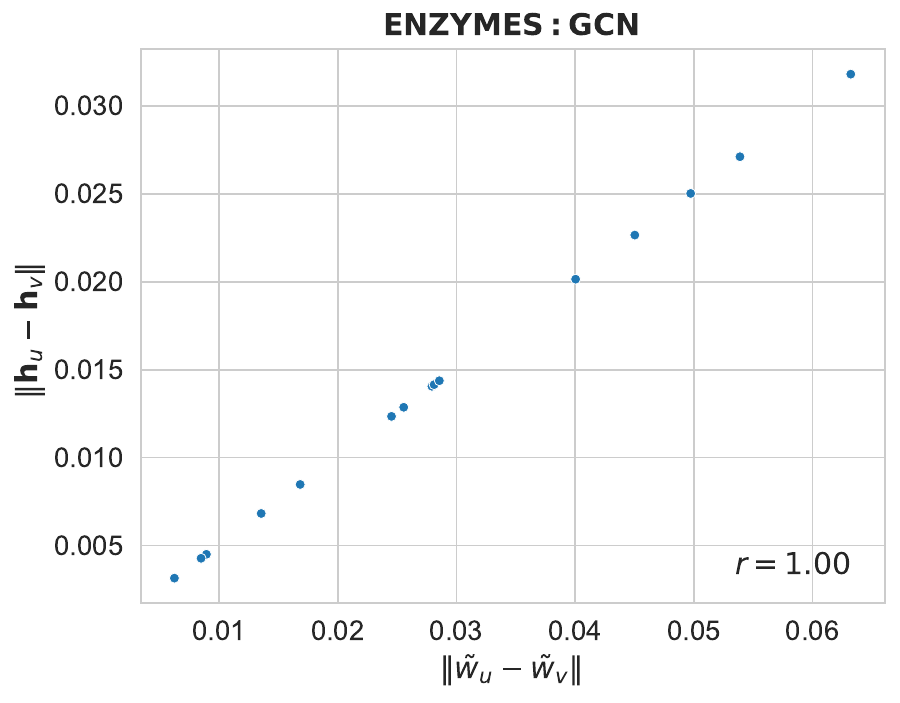}
    \end{subfigure}
    \caption{Euclidean distances of the representations generated at the third layer of the different models vs. Euclidean distances of the number of walks (or sum of normalized walks) of length $3$ starting from the different nodes.
    Nodes $v$ and $u$ correspond to the same node in the original and the perturbed graph, respectively.
    Each perturbed graph has emerged by removing one node from node $v$'s neighborhood.}
    \label{fig:corr_gnns_node}
\end{figure*}

\subsection{Are MPNNs Stable under Perturbations?}
We next capitalize on Theorem~\ref{thm:lipschitz} to assess the stability of MPNNs to small perturbations such as node removal or edge removal.
Note that the Lipschitz constant is a common tool to evaluate neural networks' stability to small perturbations.
Our theoretical result implies that the Euclidean distance between the representation of a node and the node's new representation once some perturbation is applied depends on the decrease (or increase) in the number of walks that start from that node.
We empirically validate our theoretical result on the ENZYMES dataset.
We first split the dataset into training, validation and test sets, and then train the GIN-$0$ and GCN models on the training set.
Finally, we choose one node $v$ of some graph of the test set and use the model to produce its represention $\mathbf{h}_v$ at the third layer of each model.
We then create perturbed versions of the graph.
Each perturbed graph emerges from the original graph by just removing one node (and its adjacent edges) whose shortest path distance from $v$ is at most $2$.
Let $u$ denote the node of the perturbed graph that corresponds to $v$, and $\mathbf{h}_u$ its represention at the third layer of the two models.
Figure~\ref{fig:corr_gnns_node} illustrates how the Euclidean distance between the two nodes varies as a function of the descrease in the number of walks.
We observe that the Euclidean distance between the initial representation of the node and its new representation is highly correlated with the decrease in the number of walks due to the perturbations.
Besides nodes, we also experimented with edge removal (we removed edges whose endpoints are both at distance at most $2$ from $v$), and the results are given in Appendix~\ref{sec:further_results}.

Note that there exist graphs where the removal of a node or of an edge can significantly degrade graph robustness.
Consider for example the graph that is shown in Figure~\ref{fig:example_bridge}.
The removal of node $u$ (or of one of its adjacent edges) disconnects the graph.
This will most likely have a significant impact on the representation of node $v$ computed at the $k$-th layer of GIN where $k > 4$.
\begin{wrapfigure}{r}{0.4\textwidth}
    \centering
    \scriptsize
    \begin{tikzpicture}
        \node[draw,circle,label=below:$v$] (A) at (0,0) {};
        \node[draw,circle] (B) at (0.75,0) {};
        \node[draw,circle,label=below:$u$] (C) at (1.5,0) {};
        \node[draw,circle] (D) at (2.25,0) {};
        \node[draw,circle] (E) at (3,0.75) {};
        \node[draw,circle] (F) at (3,-0.75) {};
        \node[draw,circle] (G) at (3.75,0.75) {};
        \node[draw,circle] (H) at (3.75,-0.75) {};
        \node[draw,circle,label=below:$z$] (I) at (4.5,0) {};

        \draw[-] (A) -- (B);
        \draw[-] (B) -- (C);
        \draw[-] (C) -- (D);
        \draw[-] (D) -- (E);
        \draw[-] (D) -- (F);
        \draw[-] (D) -- (G);
        \draw[-] (D) -- (H);
        \draw[-] (D) -- (I);
        \draw[-] (E) -- (F);
        \draw[-] (E) -- (G);
        \draw[-] (E) -- (H);
        \draw[-] (E) -- (I);
        \draw[-] (F) -- (G);
        \draw[-] (F) -- (H);
        \draw[-] (F) -- (I);
        \draw[-] (G) -- (H);
        \draw[-] (G) -- (I);
        \draw[-] (H) -- (I);

        \draw[-,thick,red] (1.7,0.2) -- (1.3,-0.2);
        \draw[-,thick,red] (1.7,-0.2) -- (1.3,0.2);
    \end{tikzpicture}
    \caption{Example of a graph where the removal of a node disconnects the graph.
    This leads to a large decrease in the number of walks of length $6$ that start from node $v$.}
    \label{fig:example_bridge}
    \vspace{-.4cm}
\end{wrapfigure}
Indeed we observe that the number of walks of length $k$ that start from node $v$ decreases significantly once node $u$ (or one of its adjacent edges) is removed.
There exist $603$ walks of length $6$ from node $v$ (with self-loops), but this number decreases to as few as $64$ once node $u$ is removed.
We initialized $10$ different GIN-$0$ models consisting of $6$ layers and fed the graph of Figure~\ref{fig:example_bridge} into the models (the models were not trained) along with two perturbations of the graph where either node $z$ is removed or node $u$ is removed (which results into a disconnected graph).
We computed the norm of the difference between the representation of node $v$ in the original graph and its representation in each of the perturbed instances at the sixth layer of the models.
We found that when node $z$ is removed, the average norm of the difference is equal to $0.0026$.
When node $u$ is removed, the average norm of the difference is equal to $0.0162$.
This confirms our theory since by removing node $u$ there is a large decrease in the number of walks that start from node $v$.
On the other hand, by removing node $z$, the perturbation in terms of the number of walks that start from node $v$ is not very strong, thus leading to a smaller value of the norm.

\subsection{How Do Initial Node Features Influence Representations?}
So far, we have focused on graphs which do not contain node features.
For our analysis, we assumed that all nodes are annotated with the same feature(s) which is common practice in the field of graph machine learning.
However, in many real-world applications (\eg chemo-informatics), the entities that correspond to the nodes of the graph are usually associated with features and those features are not necessarily the same across all entities.
Furthermore, to improve the models' expressive power, prior work typically annotates nodes with local or global features such as the degrees of the nodes or features extracted from spectral embeddings.
We thus next investigate whether our previous results can be generalized to the setting where the features of a node are different from those of other nodes.

Our next result generalizes the previous results to graphs that contain node features provided that those features all point in the same direction.
Formally, let $\mathcal{G} = \{G_1, \ldots, G_N\}$ be a collection of graphs, and let also $\mathcal{V} = V_1 \cup \ldots \cup V_N$ denote the set that contains the nodes of all those graphs.
Then, if for any two nodes $v,u \in \mathcal{V}$, we have that $\mathbf{h}_v^{(0)}$ is a positive scalar multiple of $\mathbf{h}_u^{(0)}$, then we can still bound the Lipschitz constant of those models with respect to the weighted sum of (normalized) walks where the weights correspond to the initial node features.
Let $\mathbf{w}_v^{(k)}$ denote the sum of weighted walks of length $k$ that start from node $v$.
The weight of a walk is equal to the feature(s) of the last visited node.
For example, if there are $3$ walks of length $1$ that start from node $v$, and the features of the last visited nodes are $[2.2, 1.5]$, $[1.1, 0.75]$ and $[3.3, 2.25]$, then $\mathbf{w}_v^{(k)} = [6.6, 4.5]$.
Note that the $3$ vectors all point in the same direction.
Let also $\tilde{\mathbf{w}}_v^{(k)}$ denote the sum of weighted normalized walks of length $k$ where each walk $(v_1, v_2, \ldots, v_k)$ is equal to $\nicefrac{\mathbf{h}_{v_k}^{(0)}}{\big( (1+d(v_2)) \ldots (1+d(v_{k-1})) \sqrt{(1+d(v_1)) (1+d(v_k))} \big)}$.
\begin{theorem}
    Let $\mathcal{G}=\{G_1, \ldots,G_N\}$ be a collection of graphs.
    Let also $\mathcal{V}=V_1 \cup \ldots \cup V_N$ denote the set that contains the nodes of all graphs.
    All nodes are initially annotated with features that point in the same direction.
    Then, after $k$ neighborhood aggregation layers:
    \begin{enumerate}[leftmargin=0.6cm]
        \item GCN maps nodes to representations related to the sum of weighted normalized walks of length $k$ starting from them:
        \begin{equation*}
            \Big| \Big| \mathbf{h}_v^{(k)} - \mathbf{h}_u^{(k)} \Big| \Big|_2 \leq \prod_{i=1}^k L_f^{(i)} \Big|\Big| \tilde{\mathbf{w}}_v^{(k)} -  \tilde{\mathbf{w}}_u^{(k)} \Big|\Big|_2 
        \end{equation*}
        where $L_f^{(i)}$ denotes the Lipschitz constant of the fully-connected layer of the $i$-th neighborhood aggregation layer.
        \item Under mild assumptions (biases of MLPs are ignored), GIN-$0$ maps nodes to representations that capture the sum of weighted walks of length $k$ starting from them:
        \begin{equation*}
            \Big| \Big| \mathbf{h}_v^{(k)} - \mathbf{h}_u^{(k)} \Big| \Big|_2 \leq \prod_{i=1}^k L_f^{(i)} \Big|\Big| \mathbf{w}_v^{(k)} -  \mathbf{w}_u^{(k)} \Big|\Big|_2 
        \end{equation*}
        where $L_f^{(i)}$ denotes the Lipschitz constant of the MLP of the $i$-th neighborhood aggregation layer.
    \end{enumerate}
    \label{thm:lipschitz_features}
\end{theorem}

\begin{figure}[t]
    \centering
    \begin{subfigure}{0.32\textwidth}
    \includegraphics[width=\textwidth]{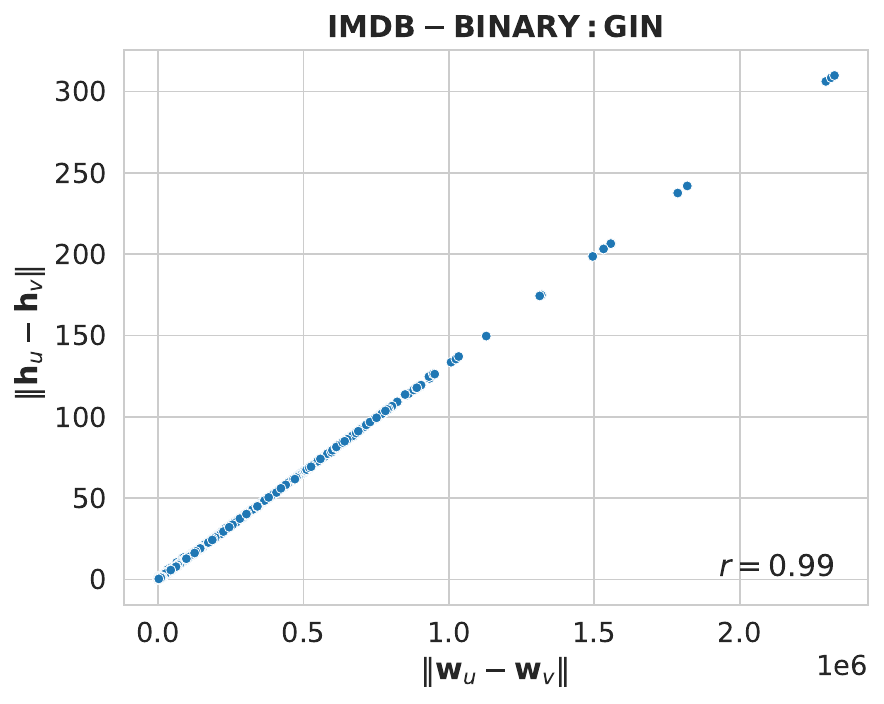}
    \end{subfigure}
    \begin{subfigure}{0.32\textwidth}
    \includegraphics[width=\textwidth]{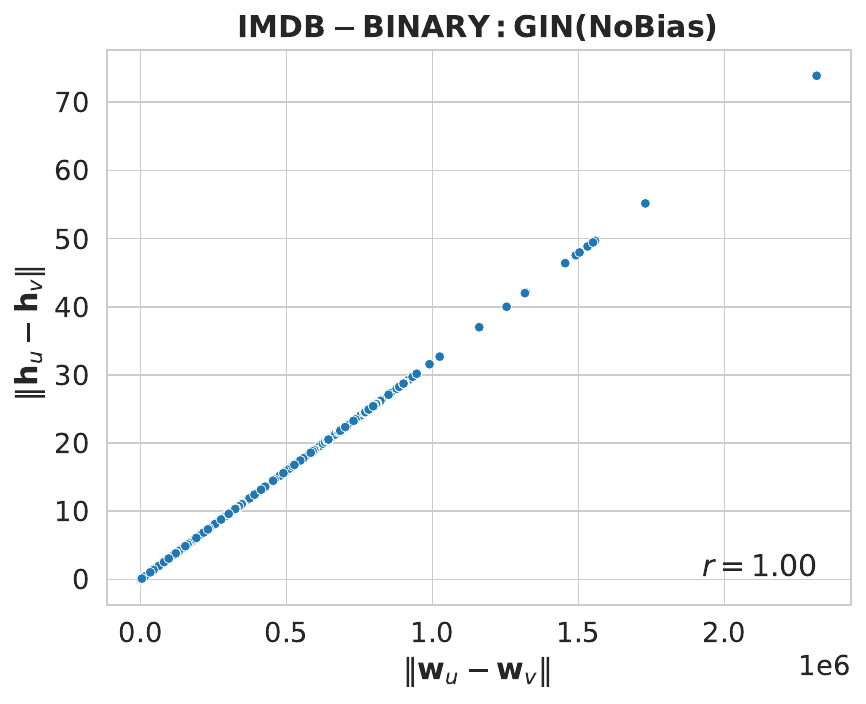}
    \end{subfigure}
    \begin{subfigure}{0.32\textwidth}
    \includegraphics[width=\textwidth]{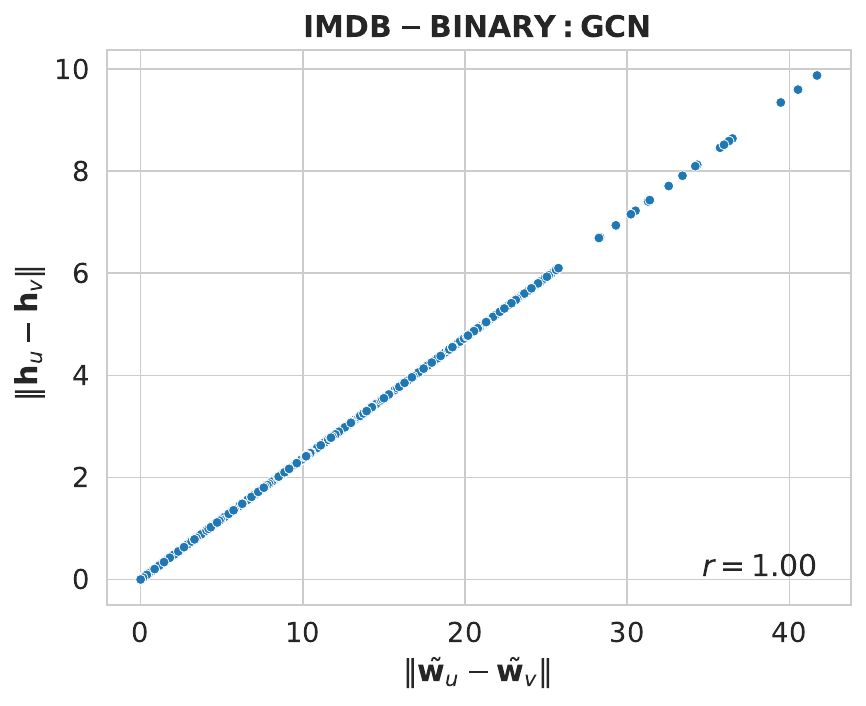}
    \end{subfigure}
    \caption{Euclidean distances of the representations generated at the third layer of the different models vs. Euclidean distances of the number of walks (or sum of normalized walks) of length $3$ starting from the different nodes.
    Each node is initially annotated with a single feature equal to its degree.}
    \label{fig:corr_gnns_features}
\end{figure}

To experimentally verify the above theoretical results, we annotated the nodes of all graphs of the IMDB-BINARY dataset with their degrees.
Note that all node features point in the same direction.
We then trained the GIN-$0$ and GCN models on that dataset.
For pairs of nodes from graphs that belong to the test set, we computed the Euclidean distance of the sum of weighted walks (resp. sum of weighted normalized walks) of length $3$ starting from them.
We also computed the Euclidean distance of the representations of the nodes that emerge at the corresponding (\ie third) layer of GIN-$0$ (resp. GCN).
We finally computed the correlation of the two collections of Euclidean distances and the results are given in Figure~\ref{fig:corr_gnns_features}.
We observe that the empirical results verify our theoretical results.
The distance of the sum of weighted walks is perfectly correlated with the distance of the representations generated by GIN-$0$ with no biases, while the distance of the sum of weighted normalized walks is perfectly correlated with the distance of the representations produced by GCN.
We also computed the Euclidean distance of the representations of the nodes that emerge at the third layer of the standard GIN-$0$ model (with biases), and we compared them against the distances of the sum of weighted walks.
We found that still there is almost perfect correlation between the two quantities.
If the features of the nodes do not point in the same direction, then our results do not hold anymore and other techniques need to be employed to derive bounds.
Deriving bounds for such kind of features is left as future work.

\section{The Phenomenon of Oversquashing from Another Perspective}
Our theoretical results are also related to the phenomenon of oversquashing~\citep{alon2021bottleneck,topping2021understanding} which occurs in MPNNs due to large information compression through bottlenecks.
Specifically, messages that are propagated from distant nodes through certain bottlenecks of the
graph, turn out to have negligible impact on the root node's representation.
Our theoretical results suggest that given two nodes $v$ and $u$, the smaller the value of $\mathbf{w}_{vu}^{(k)}$ (where $\mathbf{w}_{vu}^{(k)}$ denotes the number of (weighted) walks of length $k$ from node $v$ to node $u$), the less the impact of the message(s) from node $u$ to node $v$ on node $v$'s representation.
This becomes more severe in case the norm of $\mathbf{w}_{v}^{(k)}$ is large.

To verify our claim, we constructed a graph classification task to investigate whether an MPNN model can capture the interaction between two nodes.
All the generated graphs are instances of a single family of graphs.
Specifically, each graph consists of two components: ($1$) a complete graph with $n$ nodes; and ($2$) a perfectly balanced $r$-ary tree of height $2$.
The two components are connected by an edge, between one of the nodes of the complete graph and the root of the tree.
\begin{wrapfigure}{r}{0.3\textwidth}
    \centering
    \includegraphics[width=.2\textwidth]{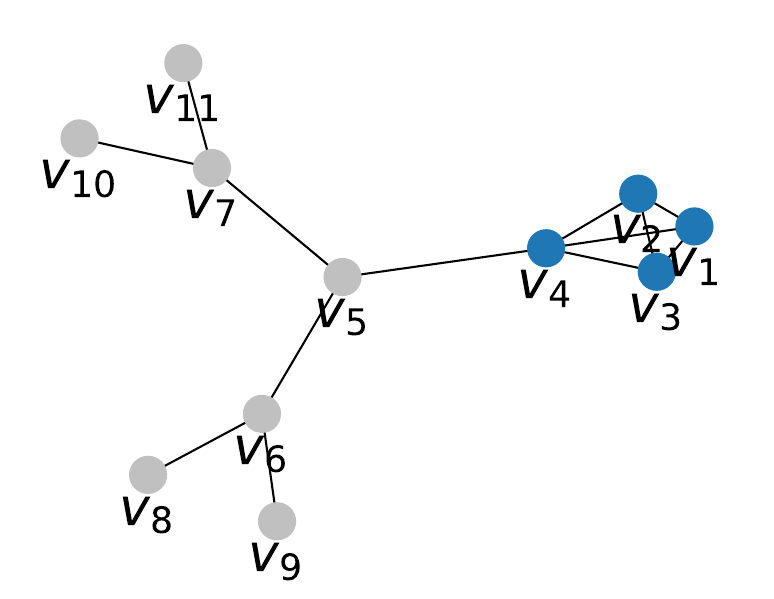}
    \caption{An example of the $\text{CBT}(4, 2)$ graph.}
    \label{fig:cbt}
    \vspace{-.4cm}
\end{wrapfigure}
We use $\text{CBT}(n, r)$ to denote such a graph with parameters $n$ and $r$.
Figure~\ref{fig:cbt} illustrates the $\text{CBT}(4, 2)$ graph.
We create a dataset that contains $\text{CBT}(n, r)$ graphs where $n$ and $r$ take the following values: $n \in \{ 4,6,\ldots, 19\}$ and $r \in \{2,3,\ldots,9 \}$.
In fact, for each combination of $n$ and $r$, we create two copies of the $\text{CBT}(n, r)$ graph.
The one copy belongs to class $0$ and the other copy belongs to class $1$.
The two graphs differ in the node feature of a single node.
While all nodes of the first graph are annotated with a feature equal to $1$, one of the leaves of the $r$-ary tree of the second graph is annotated with a feature equal to $c$.
This gives rise to $256$ graphs in total.
We split the dataset into a training, a validation and a test set with a ratio of $80\%:10\%:10\%$, and then train the GIN-$0$ model (with biases) on the first set.
\begin{figure}[t]
    \centering
    \begin{subfigure}[c]{0.37\textwidth}
        \includegraphics[width=\textwidth]{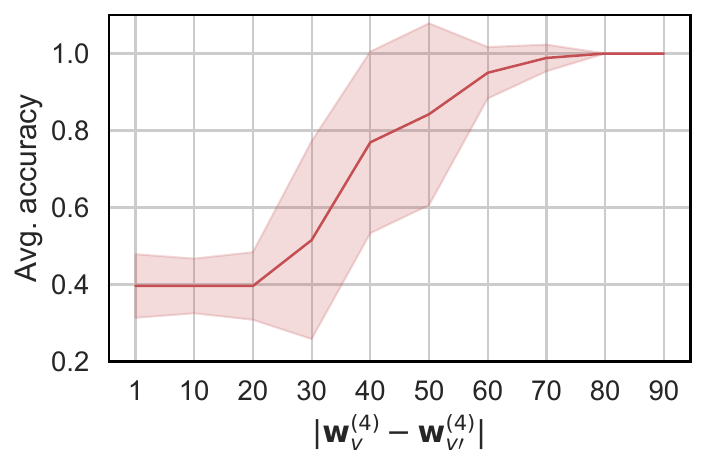}
    \end{subfigure}
    \hspace{1cm}
    \begin{subfigure}[c]{0.37\textwidth}
        \includegraphics[width=\textwidth]{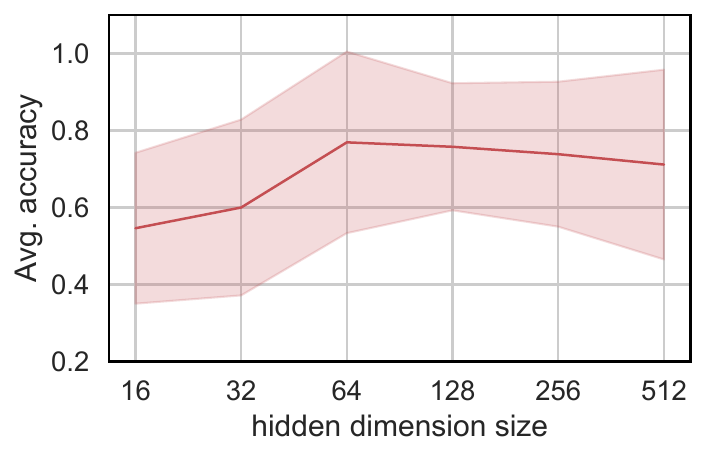}
    \end{subfigure}
    \caption{\textbf{Left}: Average accuracy on the test set of the synthetic dataset as a function of the difference of the sum of weighted walks emanating from nodes $v$ and $v'$. The two nodes correspond to structurally identical nodes of the two instances of the $\text{CBT}(n, r)$ graph. \textbf{Right}: Average accuracy on the test set of the synthetic dataset as a function of the hidden dimension size of the GIN model.}
    \label{fig:oversquashing}
\end{figure}
We set the number of neighborhood aggregation layers to $4$ (\ie shortest path distance between nodes that interact with each other).
To update the node features, we use in each neighborhood aggregation layer a multi-layer perceptron that consists of $2$ fully connected layers.
Each fully connected layer is followed by the ReLU activation function, while batch normalization is applied to the node representations that are produced by the first fully connected layer.
To make a prediction, we feed the representation of one of the nodes of the complete graph whose degree is equal to $n-1$ to a two-layer MLP.
We set the hidden dimension size to $64$.
We train the model for $500$ epochs and use the model that achieved the lowest loss on the validation set to make predictions for the test samples.
We repeat each experiment $10$ times and we report the average accuracies on the test sets.
Note that we set the feature $c$ of a single leaf node to different values (\eg $2$, $11$, etc.) and provide different results for each one of these values.
For given $n$ and $r$, let $v$ denote one node of the complete graph (where $d(v) = n-1$) in the first copy of the $\text{CBT}(n, r)$ graph, and $v'$ denote the corresponding node in the second copy.
Then, we have that $|\mathbf{w}_{v'}^{(4)} - \mathbf{w}_{v}^{(4)}| = c$.
Theorem~\ref{thm:lipschitz_features} implies that the representations of nodes $v$ and $v'$ of the two copies (that belong to different classes) will be close to each other in case $c$ is small.

The results are provided in Figure~\ref{fig:oversquashing} (Left).
We observe that when $|\mathbf{w}_{v'}^{(4)} - \mathbf{w}_{v}^{(4)}|$ is small, the model achieves very low levels of performance.
This is not surprising since for each combination of $n$ and $r$, there are two nodes ($v$ and $v'$) that belong to different classes, but their representations at the fourth layer of the model are very similar to each other.
It is very hard then for the MLP that produces the output to distinguish between samples of class $0$ and samples of class $1$. 
The model's performance increases as $|\mathbf{w}_{v'}^{(4)} - \mathbf{w}_{v}^{(4)}|$ increases.
For $|\mathbf{w}_{v'}^{(4)} - \mathbf{w}_{v}^{(4)}| \geq 80$, the model correctly classifies all test samples. 
Our results provide a different perspective for the phenomenon of oversquashing.
They indicate that it predominantly arises when the number of walks from one node to some other node is disproportionally small compared to the total walks originating from the former.
We also investigate what is the impact of the model's hidden dimension size on its ability to capture the dependence between the two nodes.
We set $c = 41$ (\ie $|\mathbf{w}_{v'}^{(4)} - \mathbf{w}_{v}^{(4)}| = 40$), and we compute the model's average accuracy on the test sets as a function of the hidden dimension size.
The results are shown in Figure~\ref{fig:oversquashing} (Right).
While for small values, it appears that the increase in the hidden dimension size also improves the model's performance, there is no further increase in performance for hidden dimension sizes beyond $64$.
Our results thus indicate that this limitation of MPNNs cannot be addressed just by increasing the hidden dimension size and that other approaches need to be employed.

\section{Conclusion}
In this paper, we focused on four well-established MPNN models and investigated what properties of graphs these models can capture.
First, we considered the case where no node features are available.
We found that two models capture no structural properties of graphs, while the rest of the models learn node representations that capture the sum of (normalized) walks emanating from the different nodes.
We established Lipschitz bounds for these models with respect to the sum of (normalized) walks.
We generalized the above results in case of node features that point in the same direction.
Finally, we provided a different perspective for the phenomenon of oversquashing.





\bibliographystyle{unsrtnat}
\bibliography{reference}

\appendix
\section{Proof of Theorem~\ref{thm:lipschitz}}

We assume that all the nodes of the graph are initially annotated with a single feature equal to $1$.

\subsection{DGCNN}
The DGCNN model updates node representations as follows~\cite{zhang2018end}:
\begin{equation*}
    \begin{split}
        \mathbf{h}_v^{(k)} = f \Bigg( \sum_{u \in \mathcal{N}(v) \cup \{ v\}} \frac{1}{d(v)+1} \mathbf{W}^{(k)} \mathbf{h}_u^{(k-1)} \Bigg) 
    \end{split}
\end{equation*}
We will show by induction that the model maps all nodes to the same vector representation.
We first assume that $\mathbf{h}_v^{(k-1)} = \mathbf{h}_u^{(k-1)} = \mathbf{h}^{(k-1)}$ for all $v,u \in V$.
This is actually true for $k=1$ since $\mathbf{h}_v^{(0)} = 1$ for all $v \in V$.
Then, we have that $\mathbf{W}^{(k)} \mathbf{h}_v^{(k-1)} = \mathbf{W}^{(k)} \mathbf{h}_u^{(k-1)} = \mathbf{W}^{(k)} \mathbf{h}^{(k-1)}$ for all $v,u \in V$.
We also have that:
\begin{equation*}
    \sum_{u \in \mathcal{N}(v) \cup \{ v\}} \frac{1}{d(v)+1} = 1 
\end{equation*}
Thus, we finally have that:
\begin{equation*}
    \begin{split}
        \mathbf{h}_v^{(k)} = f \Bigg( \sum_{u \in \mathcal{N}(v) \cup \{ v\}} \frac{1}{d(v)+1} \mathbf{W}^{(k)} \mathbf{h}_u^{(k-1)} \Bigg) = f \Bigg( \mathbf{W}^{(k)} \mathbf{h}^{(k-1)} \Bigg)
    \end{split}
\end{equation*}
for all $v \in V$.
We have shown that this variant of the GCN model produces the same representation for all nodes of all graphs and thus, it cannot capture any structural information about the neighborhood of each node.

\subsection{GAT}
The GAT model updates node representations as follows~\cite{velivckovic2018graph}:
\begin{equation*}
    \begin{split}
        \mathbf{h}_v^{(k)} = \sigma \Bigg( \sum_{u \in \mathcal{N}(v)} \alpha_{vu} \mathbf{W}^{(k)} \mathbf{h}_u^{(k-1)} \Bigg) 
    \end{split}
\end{equation*}
where $\alpha_{vu}$ is an attention coefficient that indicates the importance of node $u$'s features to node $v$.
Once again, we assume that all the nodes of the graph are initially annotated with a single feature equal to $1$.
We will show by induction that the model maps all nodes to the same vector representation.
We first assume that $\mathbf{h}_v^{(k-1)} = \mathbf{h}_u^{(k-1)} = \mathbf{h}^{(k-1)}$ for all $v,u \in V$.
This is actually true for $k=1$ since $\mathbf{h}_v^{(0)} = 1$ for all $v \in V$.
Then, we have that $\mathbf{W}^{(k)} \mathbf{h}_v^{(k-1)} = \mathbf{W}^{(k)} \mathbf{h}_u^{(k-1)} = \mathbf{W}^{(k)} \mathbf{h}^{(k-1)}$ for all $v,u \in V$.
Since the attention coefficients are normalized, we have:
\begin{equation*}
    \sum_{u \in \mathcal{N}(v)} \alpha_{vu} = 1 
\end{equation*}
Thus, we finally have that:
\begin{equation*}
    \begin{split}
        \mathbf{h}_v^{(k)} = \sigma \Bigg( \sum_{u \in \mathcal{N}(v)} \alpha_{vu} \mathbf{W}^{(k)} \mathbf{h}_u^{(k-1)} \Bigg) = \sigma \Bigg( \mathbf{W}^{(k)} \mathbf{h}^{(k-1)} \Bigg)
    \end{split}
\end{equation*}
for all $v \in V$.
We have shown that the GAT model produces the same representation for all nodes of all graphs and thus, it cannot capture any structural information about the neighborhood of each node.

\subsection{GCN}
The GCN model updates node representations as follows~\cite{kipf2017semi}:
\begin{equation*}
    \mathbf{h}_v^{(k)} = \text{ReLU} \Bigg( \sum_{u \in \mathcal{N}(v) \cup \{v\}} \frac{\mathbf{W}^{(k)} \mathbf{h}_u^{(k-1)}}{\sqrt{(1+d(v)) (1+d(u))}} \Bigg) 
\end{equation*}
Note that the GCN model (as decribed in the original paper~\cite{kipf2017semi}) does not perform an affine transformation of the node features, but instead a linear transformation.
In other words, no biases are present.
Thus, the following holds:
\begin{equation*}
    \begin{split}
        \mathbf{h}_v^{(k)} &= \text{ReLU} \Bigg( \sum_{u \in \mathcal{N}(v) \cup \{v\}} \frac{\mathbf{W}^{(k)} \mathbf{h}_u^{(k-1)}}{\sqrt{(1+d(v)) (1+d(u))}} \Bigg) \\
        &= \text{ReLU} \Bigg( \mathbf{W}^{(k)} \sum_{u \in \mathcal{N}(v) \cup \{v\}}  \frac{\mathbf{h}_u^{(k-1)}}{\sqrt{(1+d(v)) (1+d(u))}} \Bigg) 
    \end{split}
\end{equation*}
We next prove the following Lemma which will be useful in our analysis.
\begin{lemma}
    Let $\mathbf{W}$ denote a matrix.
    Let also $\mathcal{X} = \{ \mathbf{x}_1, \ldots, \mathbf{x}_M \}$ denote a set of vectors such that given any two vectors from the set, one is a positive scalar multiple of the other, \ie if $\mathbf{x}_i, \mathbf{x}_j \in \mathcal{X}$, then $\mathbf{x}_i = a \mathbf{x}_j$ where $a>0$.
    Let also $c_i > 0$ for all $i \in \{1,\ldots,M\}$. 
    Then, the following holds:
    \begin{equation*}
       \sum_{i=1}^M c_i \text{ReLU}(\mathbf{W} \, \mathbf{x}_i) = \text{ReLU} \bigg( \mathbf{W} \sum_{i=1}^M c_i \mathbf{x}_i \bigg)
    \end{equation*}
    \label{lemma:1}
\end{lemma}
\begin{proof}
    For any scalar $a$, we have that $\mathbf{W} (a \mathbf{x}) = a \mathbf{W} \mathbf{x}$.
    Furthermore, for $a > 0$, we have that $\text{ReLU}(\mathbf{W} (a \mathbf{x})) = \text{ReLU}(a \mathbf{W} \mathbf{x}) = a \text{ReLU}(\mathbf{W} \mathbf{x})$.
    Suppose that $\mathbf{x}_2 = a_2 \mathbf{x}_1$, $\mathbf{x}_3 = a_3 \mathbf{x}_1, \ldots$, $\mathbf{x}_M = a_M \mathbf{x}_1$.
    Then, we have that:
    \begin{equation*}
        \begin{split}
            \sum_{i=1}^M c_i \text{ReLU}(\mathbf{W} \, \mathbf{x}_i) &= c_1 \text{ReLU}(\mathbf{W} \,\mathbf{x}_1) + \sum_{i=2}^M c_i \text{ReLU}(\mathbf{W} \, a_i \mathbf{x}_1) \\
            &= c_1 \text{ReLU}(\mathbf{W} \, \mathbf{x}_1) + \sum_{i=2}^M c_i a_i \text{ReLU}(\mathbf{W} \, \mathbf{x}_1) \\
            &= (c_1+a_2 c_2 +\ldots+a_M c_M) \text{ReLU}(\mathbf{W} \, \mathbf{x}_1) \\
            &= \text{ReLU}\Big((c_1+a_2 c_2+\ldots+a_M c_M) \mathbf{W} \, \mathbf{x}_1 \Big) \\
            &= \text{ReLU}\Big(\mathbf{W} (c_1+a_2 c_2+\ldots+a_M c_M) \mathbf{x}_1 \Big) \\
            &= \text{ReLU} \bigg( \mathbf{W} \, \sum_{i=1}^M c_i \mathbf{x}_i \bigg)
        \end{split}
    \end{equation*}
    which concludes the proof.
\end{proof}

We also prove the following Lemma.
\begin{lemma}
    Let $\mathcal{V}$ denote the set of nodes of all graphs and let $\mathcal{X}^{(k-1)} = \oms \mathbf{h}_1^{(k-1)}, \ldots, \mathbf{h}_{|\mathcal{V}|}^{(k-1)} \cms$ be the multiset of node representations that emerged at the $(k-1)$-th layer of the GCN model.
    Suppose that given any two vectors from this multiset, one is a scalar multiple of the other, \ie if $\mathbf{h}_i^{(k-1)}, \mathbf{h}_j^{(k-1)} \in \mathcal{X}^{(k-1)}$, then $\mathbf{h}_i^{(k-1)} = a \mathbf{h}_j^{(k-1)}$ where $a>0$.
    Then, the same holds for the node representations that emerge at the $k$-th layer of the model, \ie for any two vectors $\mathbf{h}_i^{(k)}, \mathbf{h}_j^{(k)} \in \mathcal{X}^{(k)} = \oms \mathbf{h}_1^{(k)}, \ldots, \mathbf{h}_{|\mathcal{V}|}^{(k)} \cms$, we have that $\mathbf{h}_i^{(k)} = a \mathbf{h}_j^{(k)}$ where $a>0$.
    \label{lemma:2}
\end{lemma}
\begin{proof}
    For $a > 0$, we have that $\mathbf{W} (a \mathbf{x}) = a \mathbf{W} \mathbf{x}$.
    Furthermore, we also have that $\text{ReLU}(a \mathbf{x}) = a \text{ReLU}(\mathbf{x})$.
    We have assumed that given an arbitrary node $w \in \mathcal{V}$, then for any node $u \in \mathcal{V}$, we have that $\mathbf{h}_u^{(k-1)} = a_u \mathbf{h}_w^{(k-1)}$.
    Then, given any node $v \in \mathcal{V}$, its representation is updated as follows:
    \begin{equation*}
        \begin{split}
            \mathbf{h}_v^{(k)} = \text{ReLU} \Bigg( \sum_{u \in \mathcal{N}(v) \cup \{v\}} \frac{\mathbf{W}^{(k)} \mathbf{h}_u^{(k-1)}}{\sqrt{(1+d(v)) (1+d(u))}} \Bigg) &= \text{ReLU} \Bigg( \sum_{u \in \mathcal{N}(v) \cup \{v\}} \frac{\mathbf{W}^{(k)} a_u \mathbf{h}_w^{(k-1)}}{\sqrt{(1+d(v)) (1+d(u))}} \Bigg) \\
            &= \sum_{u \in \mathcal{N}(v) \cup \{v\}} a_u \text{ReLU} \Bigg( \frac{\mathbf{W}^{(k)} \mathbf{h}_w^{(k-1)}}{\sqrt{(1+d(v)) (1+d(u))}} \Bigg) \\
            &= \text{ReLU} \Bigg( \frac{\mathbf{W}^{(k)} \mathbf{h}_w^{(k-1)}}{\sqrt{(1+d(v)) (1+d(u))}} \Bigg) \\
        \end{split}
    \end{equation*}
    where $c = \sum_{u \in \mathcal{N}(v) \cup \{v\}} a_u > 0$.
    Since node $v$ is an arbitrary node, we observe that the representations of all nodes point in the same direction.
\end{proof}
Thus, the above Lemma suggests that the node representations produced by GCN are either scalar multiples of each other and they point in the same direction or are all equal to the all-zeros vector.

Let $L_f^{(k)}$ denote the Lipschitz constant associated with the fully connected layer of the $k$-th neighborhood aggregation layer of GCN.
In what follows, we also denote $\mathcal{N}(v) \cup \{v\}$ as $\mathcal{N}^+(v)$ and $d(v)+1$ as $d_v^+$.
Then, we have:
\begin{adjustwidth}{-50pt}{0pt}
\begin{align*}
    || \mathbf{h}_v^{(1)} - \mathbf{h}_u^{(1)} ||_2 &= \Bigg|\Bigg| \text{ReLU} \Bigg( \sum_{w \in \mathcal{N}^+(v)} \frac{\mathbf{W}^{(1)} \mathbf{h}_w^{(0)}}{\sqrt{d_v^+ d_w^+}} \Bigg) - \text{ReLU} \Bigg( \sum_{w \in \mathcal{N}^+(u)} \frac{\mathbf{W}^{(1)} \mathbf{h}_w^{(0)}}{\sqrt{d_u^+ d_w^+}} \Bigg)  \Bigg|\Bigg|_2 \\
    &= \Bigg|\Bigg| \text{ReLU} \Bigg( \mathbf{W}^{(1)} \sum_{w \in \mathcal{N}^+(v)} \frac{\mathbf{h}_w^{(0)}}{\sqrt{d_v^+ d_w^+}} \Bigg) - \text{ReLU} \Bigg( \mathbf{W}^{(1)} \sum_{w \in \mathcal{N}^+(u)} \frac{\mathbf{h}_w^{(0)}}{\sqrt{d_u^+ d_w^+}} \Bigg)  \Bigg|\Bigg|_2 \\
    &\leq L_f^{(1)} \Bigg|\Bigg| \sum_{w \in \mathcal{N}^+(v)} \frac{1}{{\sqrt{d_v^+ d_w^+}}} \quad - \sum_{w \in \mathcal{N}^+(u)} \frac{1}{\sqrt{d_u^+ d_w^+}} \Bigg|\Bigg|_2 \\[5ex]
    || \mathbf{h}_v^{(2)} - \mathbf{h}_u^{(2)} ||_2 &= \Bigg|\Bigg| \text{ReLU} \Bigg( \sum_{w \in \mathcal{N}^+(v)} \frac{\mathbf{W}^{(2)} \mathbf{h}_w^{(1)}}{\sqrt{d_v^+ d_w^+}} \Bigg) - \text{ReLU} \Bigg( \sum_{w \in \mathcal{N}^+(u)} \frac{\mathbf{W}^{(2)} \mathbf{h}_w^{(1)}}{\sqrt{d_u^+ d_w^+}} \Bigg)  \Bigg|\Bigg|_2 \\
    &= \Bigg|\Bigg| \text{ReLU} \Bigg( \mathbf{W}^{(2)} \sum_{w \in \mathcal{N}^+(v)} \frac{\mathbf{h}_w^{(1)}}{\sqrt{d_v^+ d_w^+}} \Bigg) - \text{ReLU} \Bigg( \mathbf{W}^{(2)} \sum_{w \in \mathcal{N}^+(u)} \frac{\mathbf{h}_w^{(1)}}{\sqrt{d_u^+ d_w^+}} \Bigg)  \Bigg|\Bigg|_2 \\
    &\leq L_f^{(2)} \Bigg|\Bigg| \sum_{w \in \mathcal{N}^+(v)} \frac{1}{{\sqrt{d_v^+ d_w^+}}} \text{ReLU} \Bigg( \sum_{x \in \mathcal{N}^+(w)} \frac{\mathbf{W}^{(1)} \mathbf{h}_w^{(0)}}{\sqrt{d_w^+ d_x^+}} \Bigg) - \sum_{w \in \mathcal{N}^+(u)} \frac{1}{{\sqrt{d_u^+ d_w^+}}} \text{ReLU} \Bigg( \sum_{x \in \mathcal{N}^+(w)} \frac{\mathbf{W}^{(1)} \mathbf{h}_w^{(0)}}{\sqrt{d_w^+ d_x^+}} \Bigg) \Bigg|\Bigg|_2 \\
    &= L_f^{(2)} \Bigg|\Bigg| \text{ReLU} \Bigg( \mathbf{W}^{(1)} \sum_{w \in \mathcal{N}^+(v)} \sum_{x \in \mathcal{N}^+(w)} \frac{\mathbf{h}_w^{(0)}}{\sqrt{d_v^+ d_w^+} \sqrt{d_w^+ d_x^+}} \Bigg) - \text{ReLU} \Bigg( \mathbf{W}^{(1)} \sum_{w \in \mathcal{N}^+(u)} \sum_{x \in \mathcal{N}^+(w)} \frac{\mathbf{h}_w^{(0)}}{\sqrt{d_u^+ d_w^+} \sqrt{d_w^+ d_x^+}} \Bigg) \Bigg|\Bigg|_2 \\
    &\qquad \qquad \qquad \qquad \qquad \qquad \qquad \qquad \qquad \qquad \qquad \qquad \qquad \qquad \qquad \qquad \qquad \qquad \qquad \qquad (\text{due to Lemmas~\ref{lemma:1},~\ref{lemma:2}}) \\
    &\leq L_f^{(2)} L_f^{(1)} \Bigg|\Bigg| \sum_{w \in \mathcal{N}^+(v)} \sum_{x \in \mathcal{N}^+(w)} \frac{1}{\sqrt{d_v^+ d_w^+} \sqrt{d_w^+ d_x^+}} - \sum_{w \in \mathcal{N}^+(u)} \sum_{x \in \mathcal{N}^+(w)} \frac{1}{\sqrt{d_u^+ d_w^+} \sqrt{d_w^+ d_x^+}} \Bigg|\Bigg|_2 \\
    &= L_f^{(2)} L_f^{(1)} \Bigg|\Bigg| \sum_{w \in \mathcal{N}^+(v)} \sum_{x \in \mathcal{N}^+(w)} \frac{1}{d_w^+ \sqrt{d_v^+ d_x^+}} - \sum_{w \in \mathcal{N}^+(u)} \sum_{x \in \mathcal{N}^+(w)} \frac{1}{d_w^+ \sqrt{d_u^+ d_x^+}} \Bigg|\Bigg|_2 \\[5ex]
    &\quad\vdots\\[5ex]
    || \mathbf{h}_v^{(K)} - \mathbf{h}_u^{(K)} ||_2 &= \Bigg|\Bigg| \text{ReLU} \Bigg( \sum_{w \in \mathcal{N}^+(v)} \frac{\mathbf{W}^{(K)} \mathbf{h}_w^{(K-1)}}{\sqrt{d_v^+ d_w^+}} \Bigg) - \text{ReLU} \Bigg( \sum_{w \in \mathcal{N}^+(u)} \frac{\mathbf{W}^{(K)} \mathbf{h}_w^{(K-1)}}{\sqrt{d_u^+ d_w^+}} \Bigg)  \Bigg|\Bigg|_2 \\
    &= \Bigg|\Bigg| \text{ReLU} \Bigg( \mathbf{W}^{(K)} \hspace{-.1cm} \sum_{w \in \mathcal{N}^+(v)} \frac{\mathbf{h}_w^{(K-1)}}{\sqrt{d_v^+ d_w^+}} \Bigg) - \text{ReLU} \Bigg( \mathbf{W}^{(K)} \sum_{w \in \mathcal{N}^+(u)} \frac{\mathbf{h}_w^{(K-1)}}{\sqrt{d_u^+ d_w^+}} \Bigg)  \Bigg|\Bigg|_2 \\
    &\leq L_f^{(K)} \Bigg|\Bigg| \sum_{w \in \mathcal{N}^+(v)} \frac{1}{{\sqrt{d_v^+ d_w^+}}} \text{ReLU} \Bigg( \sum_{x \in \mathcal{N}^+(w)} \frac{\mathbf{W}^{(K-1)} \mathbf{h}_w^{(K-2)}}{\sqrt{d_w^+ d_x^+}} \Bigg) \\
    &\qquad \qquad \qquad \qquad \qquad \qquad - \sum_{w \in \mathcal{N}^+(u)} \frac{1}{{\sqrt{d_u^+ d_w^+}}} \text{ReLU} \Bigg( \sum_{x \in \mathcal{N}^+(w)} \frac{\mathbf{W}^{(K-1)} \mathbf{h}_w^{(K-2)}}{\sqrt{d_w^+ d_x^+}} \Bigg) \Bigg|\Bigg|_2 \\
    &= L_f^{(K)} \Bigg|\Bigg| \text{ReLU} \Bigg( \mathbf{W}^{(K-1)} \sum_{w \in \mathcal{N}^+(v)} \sum_{x \in \mathcal{N}^+(w)} \frac{\mathbf{h}_w^{(K-2)}}{\sqrt{d_v^+ d_w^+} \sqrt{d_w^+ d_x^+}} \Bigg) \\
    &\qquad \qquad \qquad \qquad \qquad \qquad - \text{ReLU} \Bigg( \mathbf{W}^{(K-1)} \sum_{w \in \mathcal{N}^+(u)} \sum_{x \in \mathcal{N}^+(w)} \frac{\mathbf{h}_w^{(K-2)}}{\sqrt{d_u^+ d_w^+} \sqrt{d_w^+ d_x^+}} \Bigg) \Bigg|\Bigg|_2 (\text{due to Lemmas~\ref{lemma:1},~\ref{lemma:2}})\\
    &\leq L_f^{(K)} L_f^{(K-1)} \ldots \, L_f^{(2)}  \Bigg|\Bigg| \sum_{w \in \mathcal{N}^+(v)} \ldots \sum_{x \in \mathcal{N}^+(y)} \frac{1}{\sqrt{d_v^+ d_w^+} \ldots \sqrt{d_x^+ d_y^+}} \text{ReLU} \Bigg( \sum_{z \in \mathcal{N}^+(x)} \frac{\mathbf{W}^{(1)} \mathbf{h}_z^{(0)}}{\sqrt{d_x^+ d_z^+}} \Bigg) \\
    &\qquad \qquad \qquad \qquad \qquad \qquad - \sum_{w \in \mathcal{N}^+(u)} \ldots \sum_{x \in \mathcal{N}^+(y)} \frac{1}{\sqrt{d_u^+ d_w^+} \ldots \sqrt{d_y^+ d_x^+}} \text{ReLU} \Bigg( \sum_{z \in \mathcal{N}^+(x)} \frac{\mathbf{W}^{(1)} \mathbf{h}_z^{(0)}}{\sqrt{d_x^+ d_z^+}} \Bigg) \Bigg|\Bigg|_2 \\
    &= L_f^{(K)} L_f^{(K-1)} \ldots \, L_f^{(2)}  \Bigg|\Bigg| \text{ReLU} \Bigg( \mathbf{W}^{(1)} \sum_{w \in \mathcal{N}^+(v)} \ldots \sum_{x \in \mathcal{N}^+(y)} \sum_{z \in \mathcal{N}^+(x)} \frac{\mathbf{h}_z^{(0)}}{\sqrt{d_v^+ d_w^+} \ldots \sqrt{d_x^+ d_y^+} \sqrt{d_x^+ d_z^+}} \Bigg) \\
    &\qquad \qquad \qquad \qquad - \text{ReLU} \Bigg( \mathbf{W}^{(1)} \sum_{w \in \mathcal{N}^+(u)} \sum_{x \in \mathcal{N}^+(y)} \sum_{z \in \mathcal{N}^+(x)} \frac{\mathbf{h}_z^{(0)}}{\sqrt{d_u^+ d_w^+} \ldots \sqrt{d_y^+ d_x^+} \sqrt{d_x^+ d_z^+}} \Bigg) \Bigg|\Bigg|_2 (\text{due to Lemmas~\ref{lemma:1},~\ref{lemma:2}})\\
    &\leq L_f^{(K)} L_f^{(K-1)} \ldots \, L_f^{(1)}  \Bigg|\Bigg| \sum_{w \in \mathcal{N}^+(v)} \ldots \sum_{x \in \mathcal{N}^+(y)} \sum_{z \in \mathcal{N}^+(x)} \frac{1}{\sqrt{d_v^+ d_w^+} \ldots \sqrt{d_y^+ d_x^+} \sqrt{d_x^+ d_z^+}} \\
    &\qquad \qquad \qquad \qquad \qquad \qquad \qquad \qquad \qquad - \sum_{w \in \mathcal{N}^+(u)} \ldots \sum_{x \in \mathcal{N}^+(y)} \sum_{z \in \mathcal{N}^+(x)} \frac{1}{\sqrt{d_u^+ d_w^+} \ldots \sqrt{d_y^+ d_x^+} \sqrt{d_x^+ d_z^+}} \Bigg|\Bigg|_2 \\
    &= L_f^{(K)} L_f^{(K-1)} \ldots \, L_f^{(1)}  \Bigg|\Bigg| \sum_{w \in \mathcal{N}^+(v)} \ldots \sum_{x \in \mathcal{N}^+(y)} \sum_{z \in \mathcal{N}^+(x)} \frac{1}{d_w^+ \ldots d_y^+ d_x^+ \sqrt{d_v^+ d_z^+}} \\
    &\qquad \qquad \qquad \qquad \qquad \qquad \qquad \qquad \qquad - \sum_{w \in \mathcal{N}^+(u)} \ldots \sum_{x \in \mathcal{N}^+(y)} \sum_{z \in \mathcal{N}^+(x)} \frac{1}{d_w^+ \ldots d_y^+ d_x^+ \sqrt{d_u^+ d_z^+}} \Bigg|\Bigg|_2 \\
\end{align*}
\end{adjustwidth}
It turns out that the node representations learned at the $k$-th layer of a GCN model are related to the sum of normalized walks of length $k$ starting from each node.
Given a walk of length $k$ consisting of the following nodes $(v_1, v_2, \ldots, v_k)$, the walk is normalized by the product of the degrees of nodes $v_2,\ldots, v_{k-1}$ each increased by $1$ and of the square roots of the degrees of nodes $v_1$ and $v_k$ also increased by $1$.
Thus, the contribution of each walk is inversely proportional to the degrees of the nodes of which the walk is composed.

\subsection{GIN-$0$}
The GIN-$0$ model updates node representations as follows~\cite{xu2019how}:
\begin{equation}
    \mathbf{h}_v^{(k)} = \text{MLP}^{(k)} \Bigg( \Big( 1 + \epsilon^{(k)} \Big) \mathbf{h}_v^{(k-1)} + \sum_{u \in \mathcal{N}(v)} \mathbf{h}_u^{(k-1)} \Bigg) = \text{MLP}^{(k)} \Bigg( \sum_{u \in \mathcal{N}(v) \cup \{v\}} \mathbf{h}_u^{(k-1)} \Bigg)
    \label{eq:gin0}
\end{equation}

We make the following assumption.
\begin{assumption}
    We assume that the biases of the fully-connected layers of all MLPs are equal to zero vectors.
\end{assumption}
Our results are also valid when given fully-connected layers of the form $\text{fc}(\mathbf{x}) = \mathbf{W} \mathbf{x} + \mathbf{b}$, the following holds $||\mathbf{b}||  \ll || \mathbf{W} \mathbf{x} ||$.
Note that the activation function of the MLPs of the GIN-$0$ model is the ReLU function~\cite{xu2019how}.
Without loss of generality, we assume that the MLPs consist of two fully-connected layers.
Given the above, the update function of the GIN-$0$ model takes the following form:
\begin{equation}
    \text{MLP}^{(k)}(\mathbf{x}) = \text{ReLU} \Bigg( \mathbf{W}_2^{(k)} \, \text{ReLU} \bigg( \mathbf{W}_1^{(k)} \mathbf{x} \bigg) \Bigg)
    \label{eq:mlp}
\end{equation}
We next prove the following Lemma which will be useful in our analysis.
\begin{lemma}
    Let $\text{MLP}$ denote the model defined in Equation~\eqref{eq:mlp} above.
    Let also $\mathcal{X} = \{ \mathbf{x}_1, \ldots, \mathbf{x}_M \}$ denote a set of vectors such that given any two vectors from the set, one is a scalar multiple of the other, \ie if $\mathbf{x}_i, \mathbf{x}_j \in \mathcal{X}$, then $\mathbf{x}_i = a \mathbf{x}_j$ where $a>0$.
    Then, the following holds:
    \begin{equation*}
       \sum_{i=1}^M \text{MLP}(\mathbf{x}_i) = \text{MLP} \bigg( \sum_{i=1}^M \mathbf{x}_i \bigg)
    \end{equation*}
    \label{lemma:3}
\end{lemma}
\begin{proof}
    For $a > 0$, we have that $\mathbf{W} (a \mathbf{x}) = a \mathbf{W} \mathbf{x}$.
    Furthermore, we also have that $\text{ReLU}(a \mathbf{x}) = a \text{ReLU}(\mathbf{x})$.
    Then, we have:
    \begin{equation*}
       \text{MLP}(a \mathbf{x}) = \text{ReLU} \Bigg( \mathbf{W}_2 \, \text{ReLU} \bigg( \mathbf{W}_1 (a \mathbf{x}) \bigg) \Bigg) = a \text{ReLU} \Bigg( \mathbf{W}_2 \, \text{ReLU} \bigg( \mathbf{W}_1 (\mathbf{x}) \bigg) \Bigg) = a \text{MLP}( \mathbf{x})
    \end{equation*}
    Suppose that $\mathbf{x}_2 = a_2 \mathbf{x}_1$, $\mathbf{x}_3 = a_3 \mathbf{x}_1, \ldots$, $\mathbf{x}_M = a_M \mathbf{x}_1$.
    Then, we have that:
    \begin{equation*}
        \begin{split}
            \sum_{i=1}^M \text{MLP}(\mathbf{x}_i) &= \text{MLP}(\mathbf{x}_1) + \sum_{i=2}^M \text{MLP}(a_i \mathbf{x}_1) \\
            &= \text{MLP}(\mathbf{x}_1) + \sum_{i=2}^M a_i \text{MLP}(\mathbf{x}_1) \\
            &= (1+a_2+\ldots+a_M) \text{MLP}(\mathbf{x}_1) \\
            &= \text{MLP}\Big((1+a_2+\ldots+a_M) \mathbf{x}_1 \Big) \\
            &= \text{MLP} \bigg( \sum_{i=1}^M \mathbf{x}_i \bigg)
        \end{split}
    \end{equation*}
\end{proof}
It is trivial to generalize the above Lemma to MLPs that contain more than two layers.
We also prove the following Lemma.
\begin{lemma}
    Let the MLPs of the GIN-$0$ model be instances of the MLP of Equation~\eqref{eq:mlp} above.
    Let $\mathcal{V}$ denote the set of nodes of all graphs and let $\mathcal{X}^{(k-1)} = \oms \mathbf{h}_1^{(k-1)}, \ldots, \mathbf{h}_{|\mathcal{V}|}^{(k-1)} \cms$ be the multiset of node representations that emerged at the $(k-1)$-th layer of the model.
    Suppose that given any two vectors from this multiset, one is a scalar multiple of the other, \ie if $\mathbf{h}_i^{(k-1)}, \mathbf{h}_j^{(k-1)} \in \mathcal{X}^{(k-1)}$, then $\mathbf{h}_i^{(k-1)} = a \mathbf{h}_j^{(k-1)}$ where $a>0$.
    Then, the same holds for the node representations that emerge at the $k$-th layer of the model, \ie for any two vectors $\mathbf{h}_i^{(k)}, \mathbf{h}_j^{(k)} \in \mathcal{X}^{(k)} = \oms \mathbf{h}_1^{(k)}, \ldots, \mathbf{h}_{|\mathcal{V}|}^{(k)} \cms$, we have that $\mathbf{h}_i^{(k)} = a \mathbf{h}_j^{(k)}$ where $a>0$.
    \label{lemma:4}
\end{lemma}
\begin{proof}
    For $a > 0$, we have that $\mathbf{W} (a \mathbf{x}) = a \mathbf{W} \mathbf{x}$.
    Furthermore, we also have that $\text{ReLU}(a \mathbf{x}) = a \text{ReLU}(\mathbf{x})$.
    Then, we have:
    \begin{equation*}
       \text{MLP}(a \mathbf{x}) = \text{ReLU} \Bigg( \mathbf{W}_2 \, \text{ReLU} \bigg( \mathbf{W}_1 (a \mathbf{x}) \bigg) \Bigg) = a \text{ReLU} \Bigg( \mathbf{W}_2 \, \text{ReLU} \bigg( \mathbf{W}_1 (\mathbf{x}) \bigg) \Bigg) = a \text{MLP}( \mathbf{x})
    \end{equation*}
    We have assumed that given an arbitrary node $w \in \mathcal{V}$, then for any node $u \in \mathcal{V}$, we have that $\mathbf{h}_u^{(k-1)} = a_u \mathbf{h}_w^{(k-1)}$.
    Then, given any node $v \in \mathcal{V}$, its representation is updated as follows:
    \begin{equation*}
        \begin{split}
            \mathbf{h}_v^{(k)} = \text{MLP}^{(k)} \Bigg( \sum_{u \in \mathcal{N}(v) \cup \{v\}} \mathbf{h}_u^{(k-1)} \Bigg) &= \text{MLP}^{(k)} \Bigg( \sum_{u \in \mathcal{N}(v) \cup \{v\}} a_u \mathbf{h}_w^{(k-1)} \Bigg) \\
            &= \sum_{u \in \mathcal{N}(v) \cup \{v\}} a_u \text{MLP}^{(k)} \Bigg( \mathbf{h}_w^{(k-1)} \Bigg) \\
            &= c \text{MLP}^{(k)} \Bigg( \mathbf{h}_w^{(k-1)} \Bigg)
        \end{split}
    \end{equation*}
    where $c =\sum_{u \in \mathcal{N}(v) \cup \{v\}} a_u  > 0$.
    Since node $v$ is an arbitrary node, we can conclude that all nodes obtain representations that are scalar multiples of each other and they point in the same direction.
\end{proof}
Thus, the above Lemma suggests that for MLPs of the form of Equation~\eqref{eq:mlp}, the node representations produced by GIN-$0$ are either scalar multiples of each other and they point in the same direction or are all equal to the all-zeros vector.

Let $L_f^{(k)}$ denote the Lipschitz constant associated with the MLP of the $k$-th neighborhood aggregation layer of GIN-$0$.
We assume that all the nodes of the graph are initially annotated with a single feature equal to $1$.
Then, we have:
\begin{adjustwidth}{-50pt}{0pt}
\begin{align*}
    || \mathbf{h}_v^{(1)} - \mathbf{h}_u^{(1)} ||_2 &= \Bigg|\Bigg| \textsc{MLP}^{(1)} \Bigg( \sum_{w \in \mathcal{N}^+(v)} \mathbf{h}_w^{(0)} \Bigg) - \textsc{MLP}^{(1)} \Bigg( \sum_{w \in \mathcal{N}^+(u)} \mathbf{h}_w^{(0)} \Bigg) \Bigg|\Bigg|_2 \\
    &\leq L_f^{(1)} \Bigg|\Bigg| \sum_{w \in \mathcal{N}^+(v)} 1 \quad - \sum_{w \in \mathcal{N}^+(u)} 1 \Bigg|\Bigg|_2 \\
    &= L_f^{(1)} \Big|\Big| d(v) - d(u) \Big|\Big|_2 \\[5ex]
    || \mathbf{h}_v^{(2)} - \mathbf{h}_u^{(2)} ||_2 &= \Bigg|\Bigg| \textsc{MLP}^{(2)} \Bigg( \sum_{w \in \mathcal{N}^+(v)} \mathbf{h}_w^{(1)} \Bigg) - \textsc{MLP}^{(2)} \Bigg( \sum_{w \in \mathcal{N}^+(u)} \mathbf{h}_w^{(1)} \Bigg) \Bigg|\Bigg|_2 \\
    &\leq L_f^{(2)} \Bigg|\Bigg| \sum_{w \in \mathcal{N}^+(v)} \textsc{MLP}^{(1)} \Bigg( \sum_{x \in \mathcal{N}^+(w)} \mathbf{h}_x^{(0)} \Bigg) \quad - \sum_{w \in \mathcal{N}^+(u)} \textsc{MLP}^{(1)} \Bigg( \sum_{x \in \mathcal{N}^+(w)} \mathbf{h}_x^{(0)} \Bigg) \Bigg|\Bigg|_2 \\
    &= L_f^{(2)} \Bigg|\Bigg| \textsc{MLP}^{(1)} \Bigg( \sum_{w \in \mathcal{N}^+(v)} \sum_{x \in \mathcal{N}^+(w)} \mathbf{h}_x^{(0)} \Bigg) \quad - \textsc{MLP}^{(1)} \Bigg( \sum_{w \in \mathcal{N}^+(u)} \sum_{x \in \mathcal{N}^+(w)} \mathbf{h}_x^{(0)} \Bigg) \Bigg|\Bigg|_2 (\text{due to Lemmas~\ref{lemma:3},~\ref{lemma:4}})\\
    &\leq L_f^{(2)} L_f^{(1)} \Bigg|\Bigg| \sum_{w \in \mathcal{N}^+(v)} \sum_{x \in \mathcal{N}^+(w)} \mathbf{h}_x^{(0)} \quad - \sum_{w \in \mathcal{N}^+(u)} \sum_{x \in \mathcal{N}^+(w)} \mathbf{h}_x^{(0)} \Bigg|\Bigg|_2 \\
    &= L_f^{(2)} L_f^{(1)} \Bigg|\Bigg| \sum_{w \in \mathcal{N}^+(v)} \sum_{x \in \mathcal{N}^+(w)} 1 \quad - \sum_{w \in \mathcal{N}^+(u)} \sum_{x \in \mathcal{N}^+(w)} 1 \Bigg|\Bigg|_2 \\
    \vdots\\
    || \mathbf{h}_v^{(K)} - \mathbf{h}_u^{(K)} ||_2 &= \Bigg|\Bigg| \textsc{MLP}^{(K)} \Bigg( \sum_{w \in \mathcal{N}^+(v)} \mathbf{h}_w^{(K-1)} \Bigg) - \textsc{MLP}^{(K)} \Bigg( \sum_{w \in \mathcal{N}^+(u)} \mathbf{h}_w^{(K-1)} \Bigg) \Bigg|\Bigg|_2 \\
    &\leq L_f^{(K)} \Bigg|\Bigg| \sum_{w \in \mathcal{N}^+(v)} \textsc{MLP}^{(K-1)} \Bigg( \sum_{x \in \mathcal{N}^+(w)} \mathbf{h}_x^{(K-2)} \Bigg) \quad - \sum_{w \in \mathcal{N}^+(u)} \textsc{MLP}^{(K-1)} \Bigg( \sum_{x \in \mathcal{N}^+(w)} \mathbf{h}_x^{(K-2)} \Bigg) \Bigg|\Bigg|_2 \\
    &= L_f^{(K)} \Bigg|\Bigg| \textsc{MLP}^{(K-1)} \Bigg( \sum_{w \in \mathcal{N}^+(v)} \sum_{x \in \mathcal{N}^+(w)} \mathbf{h}_x^{(K-2)} \Bigg) - \textsc{MLP}^{(K-1)} \Bigg( \sum_{w \in \mathcal{N}^+(u)} \sum_{x \in \mathcal{N}^+(w)} \mathbf{h}_x^{(K-2)} \Bigg) \Bigg|\Bigg|_2 (\text{due to Lemmas~\ref{lemma:3},~\ref{lemma:4}})\\
    &\leq L_f^{(K)} L_f^{(K-1)} \ldots \, L_f^{(2)} \Bigg|\Bigg| \sum_{w \in \mathcal{N}^+(v)} \sum_{x \in \mathcal{N}^+(w)} \ldots \quad \textsc{MLP}^{(1)} \Bigg( \sum_{z \in \mathcal{N}^+(y)} \mathbf{h}_z^{(0)} \Bigg) \quad \\
    &\qquad \qquad \qquad \qquad \qquad \qquad \qquad \qquad - \sum_{w \in \mathcal{N}(u) \cup \{u\}} \sum_{x \in \mathcal{N}^+(w)} \ldots \quad \textsc{MLP}^{(1)} \Bigg( \sum_{z \in \mathcal{N}^+(y)} \mathbf{h}_z^{(0)} \Bigg) \Bigg|\Bigg|_2 \\
    &= L_f^{(K)} L_f^{(K-1)} \ldots \, L_f^{(2)} \Bigg|\Bigg| \textsc{MLP}^{(1)} \Bigg( \sum_{w \in \mathcal{N}^+(v)} \sum_{x \in \mathcal{N}^+(w)} \ldots \sum_{z \in \mathcal{N}^+(y)} \mathbf{h}_z^{(0)} \Bigg) \quad \\
    &\qquad \qquad \qquad \qquad \qquad \qquad \qquad \qquad - \textsc{MLP}^{(1)} \Bigg( \sum_{w \in \mathcal{N}(u) \cup \{u\}} \sum_{x \in \mathcal{N}^+(w)} \ldots \sum_{z \in \mathcal{N}^+(y)} \mathbf{h}_z^{(0)} \Bigg) \Bigg|\Bigg|_2 (\text{due to Lemmas~\ref{lemma:3},~\ref{lemma:4}})\\
    &\leq L_f^{(K)} L_f^{(K-1)} \ldots \, L_f^{(1)} \Bigg|\Bigg| \sum_{w \in \mathcal{N}^+(v)} \sum_{x \in \mathcal{N}^+(w)} \ldots \sum_{z \in \mathcal{N}^+(y)} \mathbf{h}_z^{(0)} \quad \\
    &\qquad \qquad \qquad \qquad \qquad \qquad \qquad \qquad - \sum_{w \in \mathcal{N}^+(u)} \sum_{x \in \mathcal{N}^+(w)} \ldots \sum_{z \in \mathcal{N}^+(y)} \mathbf{h}_z^{(0)} \Bigg|\Bigg|_2 \\
    &= L_f^{(K)} L_f^{(K-1)} \ldots \, L_f^{(1)} \Bigg|\Bigg| \sum_{w \in \mathcal{N}^+(v)} \sum_{x \in \mathcal{N}^+(w)} \ldots \sum_{z \in \mathcal{N}^+(y)} 1 \quad - \sum_{w \in \mathcal{N}^+(u)} \sum_{x \in \mathcal{N}^+(w)} \ldots \sum_{z \in \mathcal{N}^+(y)} 1 \Bigg|\Bigg|_2 \\
\end{align*}
\end{adjustwidth}

We observe that for $k=1$, $d(v)$ and $d(u)$ are equal to the number of walks of length $1$ starting from nodes $v$ and $u$, respectively.
Likewise, for $k=2$, $\sum_{w \in \mathcal{N}(v) \cup \{v\}} \sum_{x \in \mathcal{N}(w) \cup \{w\}} 1$ and $\sum_{w \in \mathcal{N}(u) \cup \{u\}} \sum_{x \in \mathcal{N}(w) \cup \{w\}} 1$ are equal to the number of walks of length $2$ starting from nodes $v$ and $u$, respectively.
And more generally, for $k=K$, $\sum_{w \in \mathcal{N}(v) \cup \{v\}} \sum_{x \in \mathcal{N}(w) \cup \{w\}} \ldots \sum_{z \in \mathcal{N}(y) \cup \{y\}} 1$ and $\sum_{w \in \mathcal{N}(u) \cup \{u\}} \sum_{x \in \mathcal{N}(w) \cup \{w\}} \ldots \sum_{z \in \mathcal{N}(y) \cup \{y\}} 1$ are equal to the number of walks of length $K$ starting from nodes $v$ and $u$, respectively.
Thus, the representations learned by GIN-$0$ are related to the number of walks emanating from each node.

\section{Experimental Setup}\label{sec:experimental_setup}
In all the experiments performed on the ENZYMES, IMDB-BINARY and IMDB-MULTI, Cora and Citeseer datasets, the different GNN models were trained in the original learning tasks (graph classification for ENZYMES, IMDB-BINARY and IMDB-MULTI and node classification for Cora and Citeseer).
For datasets where nodes are annotated with some initial features (\eg ENZYMES, Cora, Citeseer), those features were not taken into account.
Each dataset is randomly split into training, validation, and test sets with an $80:10:10$
split ratio, respectively.
All models consist of a series of neighborhood aggregation layers.
For node-level tasks, the final neighborhood aggregation layer is followed by a $2$-layers MLP which produces class probabilities.
For graph-level tasks, the final neighborhood aggregation layer is followed by a readout function which computes the sum of the representations of the nodes.
The output of the readout function is passed on to a $2$-layers MLP which produces class probabilities.
For all experiments, the batch size is set equal to $64$.
Each model is trained for $300$ epochs by minimizing the cross-entropy loss.
We use the Adam optimizer for model training.
We store in the disk the parameters of the model that achieved the lowest validation loss and those parameters are retrieved once training has finished.
Then, $100$ nodes from the graphs that belong to the test set are randomly sampled and the representations of those nodes are extracted from the final neighborhood aggregation layer.
Then, pairwise distances of those nodes are computed and compared against the corresponding distances that emerge from the (normalized) walks.

\section{Additional Visualizations}\label{sec:further_results}
We provide further experimental results in Figures~\ref{fig:corr_gnns_imbd_multi},~\ref{fig:corr_gnns_cora} and~\ref{fig:corr_gnns_citeseer}.
The three Figures compare the Euclidean distance of the representations of the nodes that emerge at the third layer of GIN-$0$ (resp. GCN) against the Euclidean distance of the number of walks (resp. sum of normalized walks) of length $3$ starting from those nodes on the IMDB-MULTI, Cora and Citeseer datasets, respectively.
\begin{figure}[h]
    \centering
    \begin{subfigure}{0.32\textwidth}
        \includegraphics[width=\textwidth]{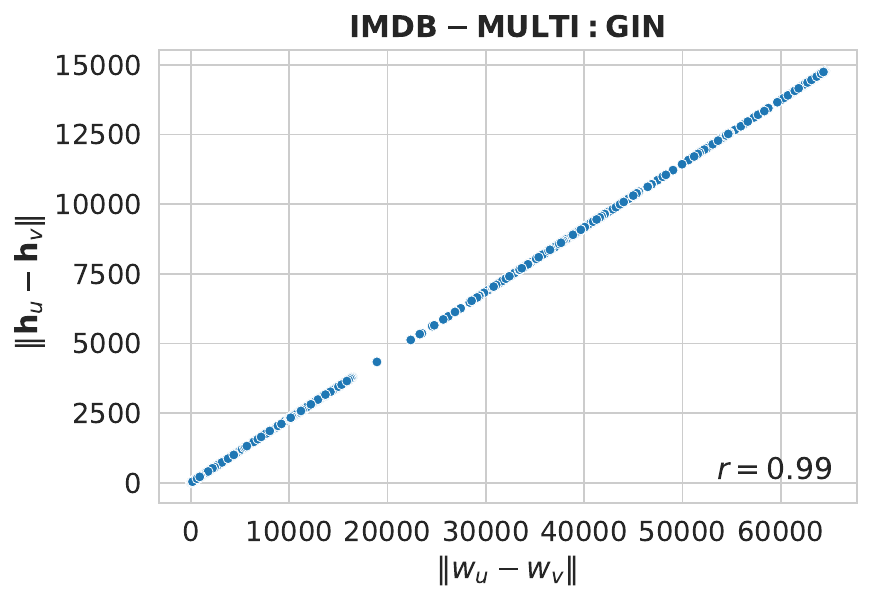}
    \end{subfigure}
    \begin{subfigure}{0.32\textwidth}
        \includegraphics[width=\textwidth]{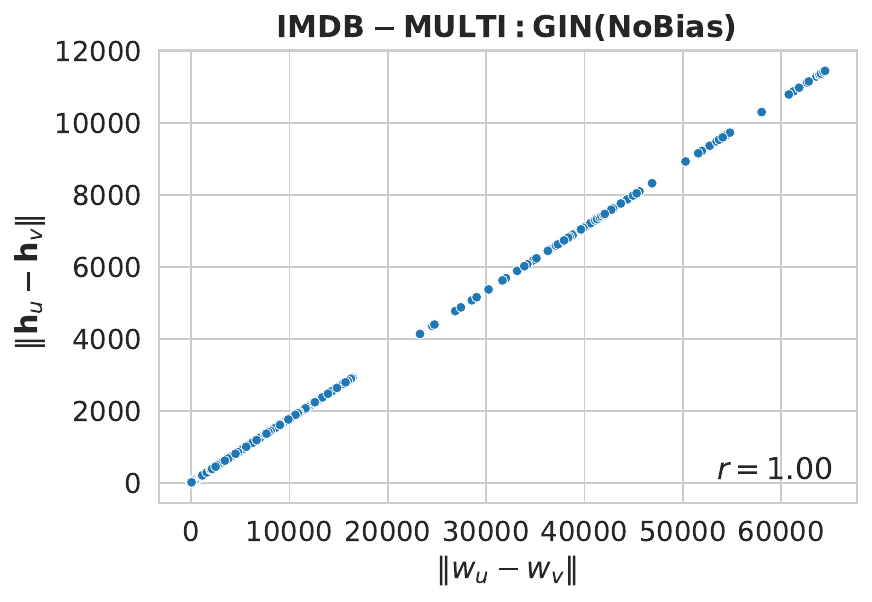}
    \end{subfigure}
    \begin{subfigure}{0.32\textwidth}
        \includegraphics[width=\textwidth]{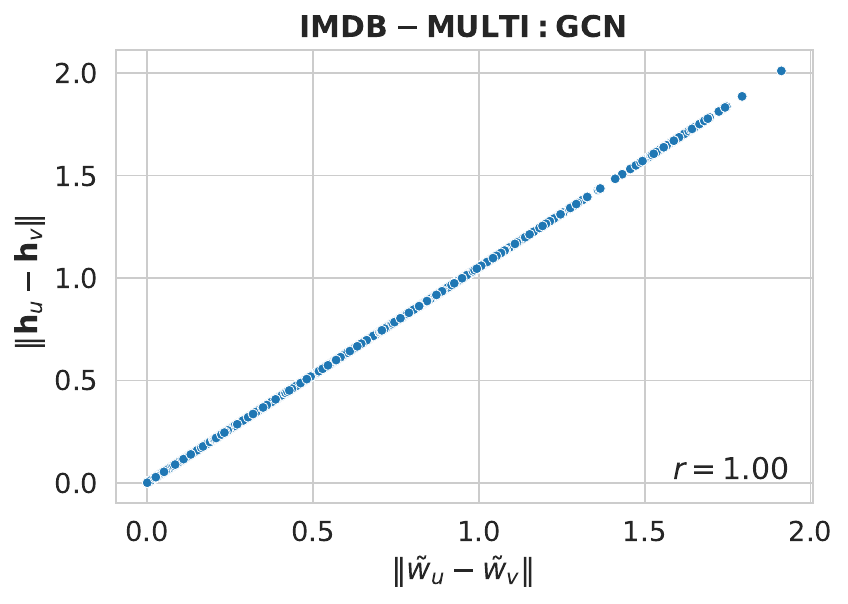}
    \end{subfigure}
    \caption{Euclidean distances of the representations generated at the third layer of the different models vs. Euclidean distances of the number of walks (or normalized walks) of length 3 starting from the different nodes on the IMDB-MULTI dataset.
    Nodes are annotated with a single feature equal to $1$.}
    \label{fig:corr_gnns_imbd_multi}
\end{figure}

\begin{figure}[h]
    \centering
    \begin{subfigure}{0.32\textwidth}
        \includegraphics[width=\textwidth]{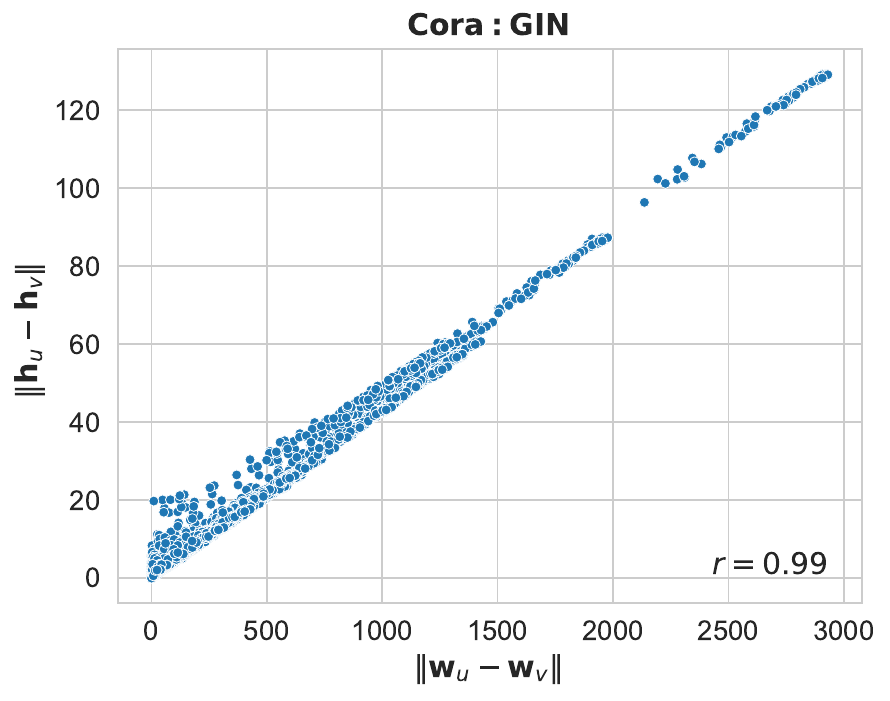}
    \end{subfigure}
    \begin{subfigure}{0.32\textwidth}
        \includegraphics[width=\textwidth]{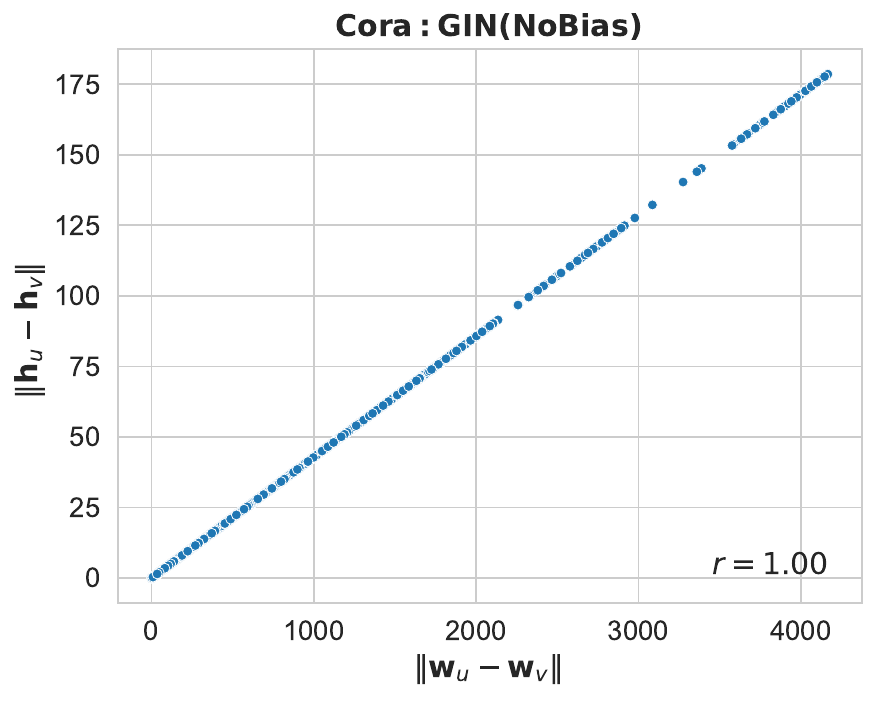}
    \end{subfigure}
    \begin{subfigure}{0.32\textwidth}
        \includegraphics[width=\textwidth]{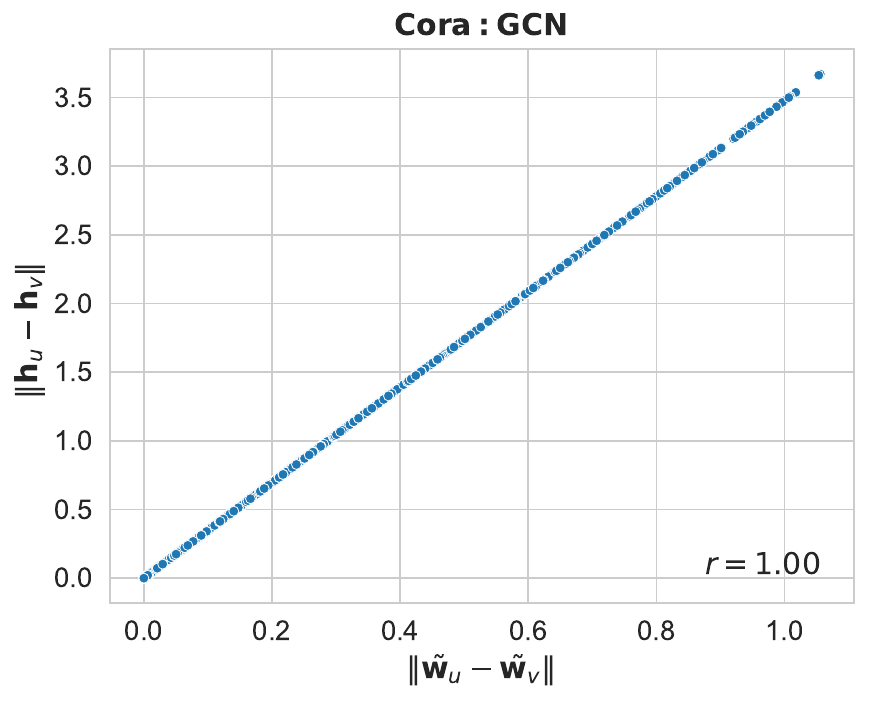}
    \end{subfigure}
    \caption{Euclidean distances of the representations generated at the third layer of the different models vs. Euclidean distances of the number of walks (or normalized walks) of length 3 starting from the different nodes on the Cora dataset.
    Nodes are annotated with a single feature equal to $1$.}
    \label{fig:corr_gnns_cora}
\end{figure}

\begin{figure}[h]
    \centering
    \begin{subfigure}{0.32\textwidth}
        \includegraphics[width=\textwidth]{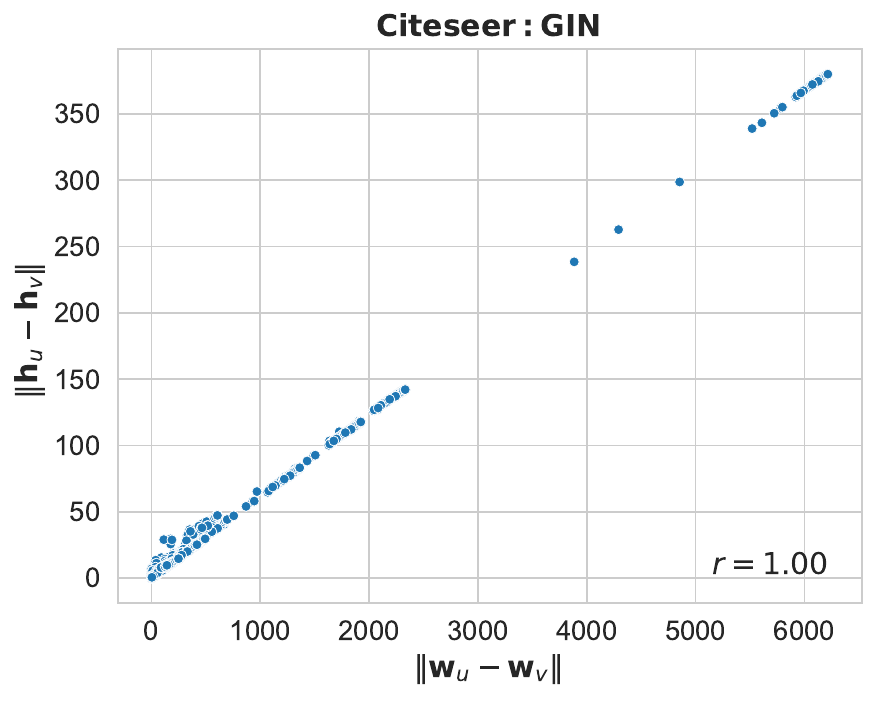}
    \end{subfigure}
    \begin{subfigure}{0.32\textwidth}
        \includegraphics[width=\textwidth]{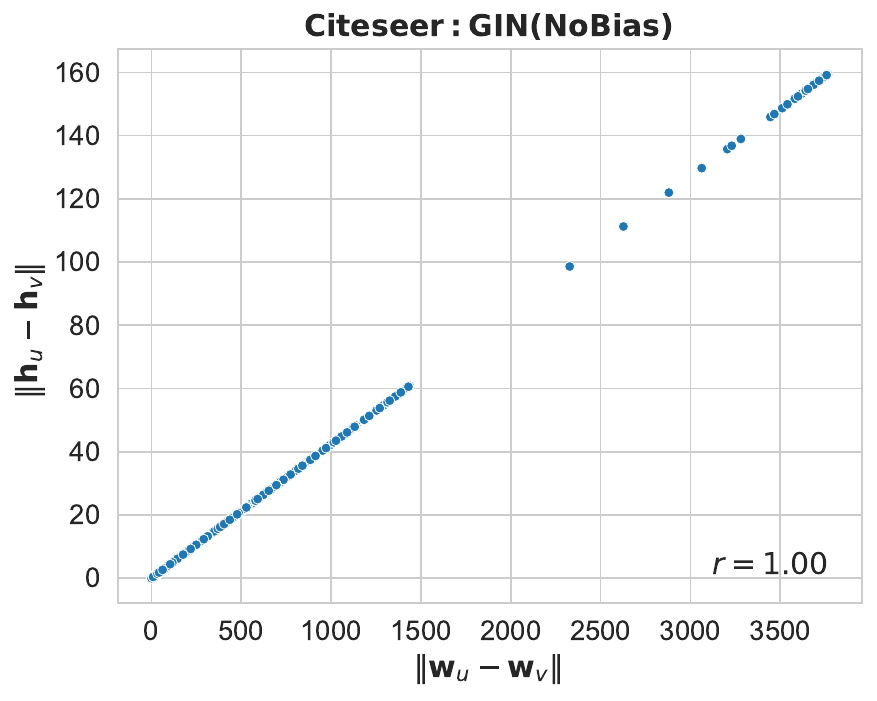}
    \end{subfigure}
    \begin{subfigure}{0.32\textwidth}
        \includegraphics[width=\textwidth]{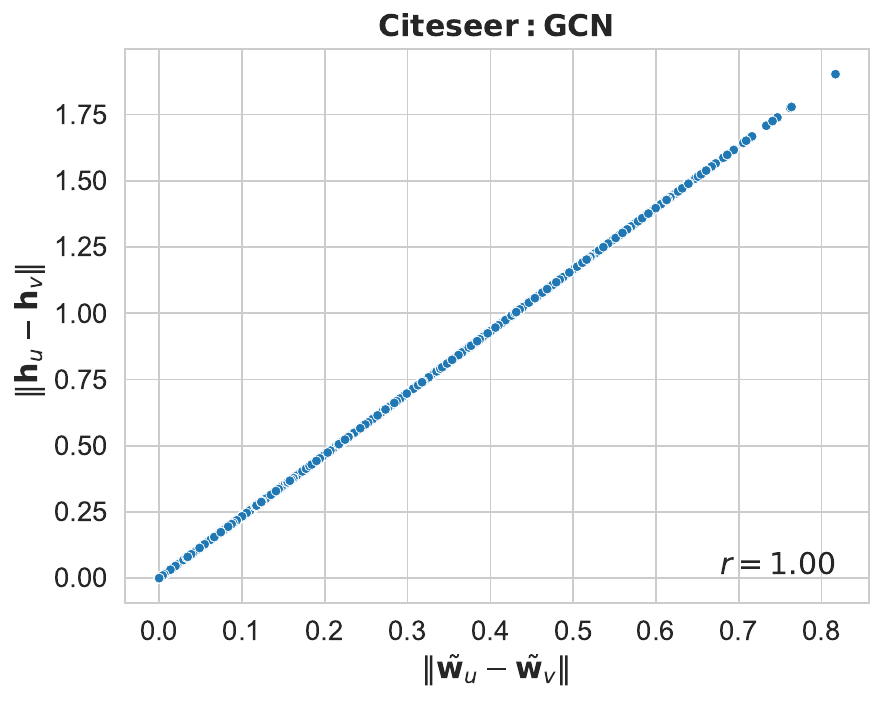}
    \end{subfigure}
    \caption{Euclidean distances of the representations generated at the third layer of the different models vs. Euclidean distances of the number of walks (or normalized walks) of length 3 starting from the different nodes on the Citeseer dataset.
    Nodes are annotated with a single feature equal to $1$.}
    \label{fig:corr_gnns_citeseer}
\end{figure}

Figure~\ref{fig:corr_gnns_edge} illustrates how the Euclidean distance between the representations of two nodes at the third layer of GIN-$0$ and GCN varies as a function of the descrease in the number of walks due to the removal of edges from the node's local neighborhood.
Each perturbed graph emerges from the original graph by just removing one edge which connects two nodes whose shortest path distance from $v$ is at most $2$.
We denote by $u$ the node of the perturbed graph that corresponds to node $v$ of the original graph.
\begin{figure*}[h]
    \begin{subfigure}{0.32\textwidth}
    \includegraphics[width=\textwidth]{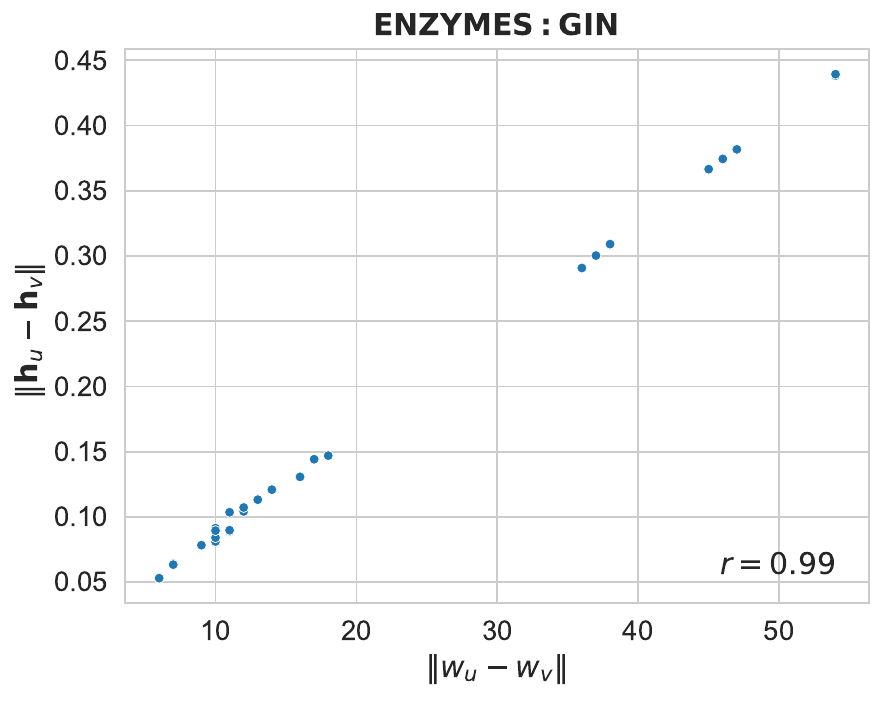}
    \end{subfigure}
    \begin{subfigure}{0.32\textwidth}
    \includegraphics[width=\textwidth]{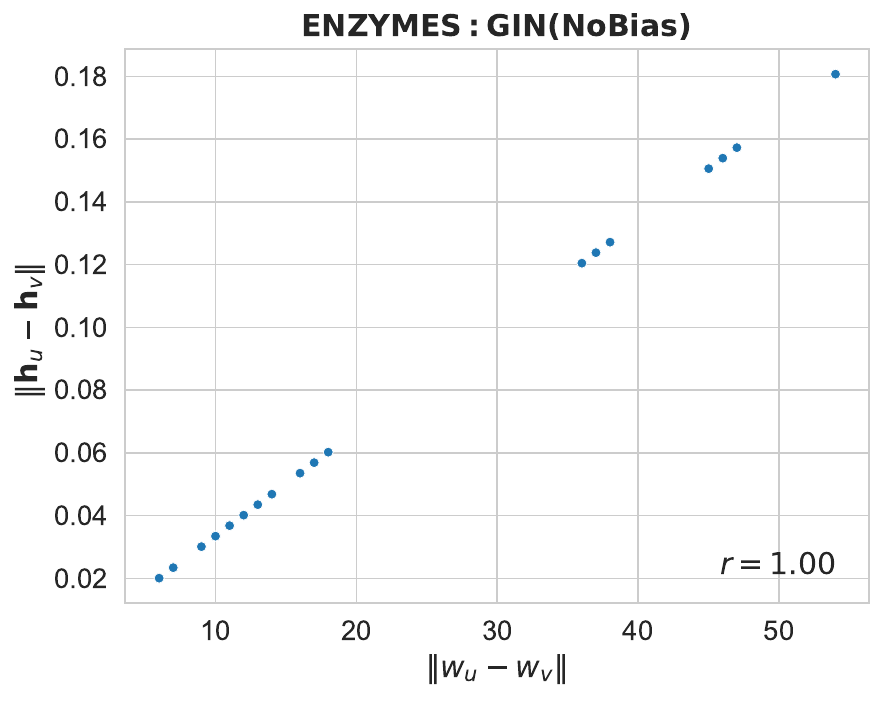}
    \end{subfigure}
    \begin{subfigure}{0.32\textwidth}
    \includegraphics[width=\textwidth]{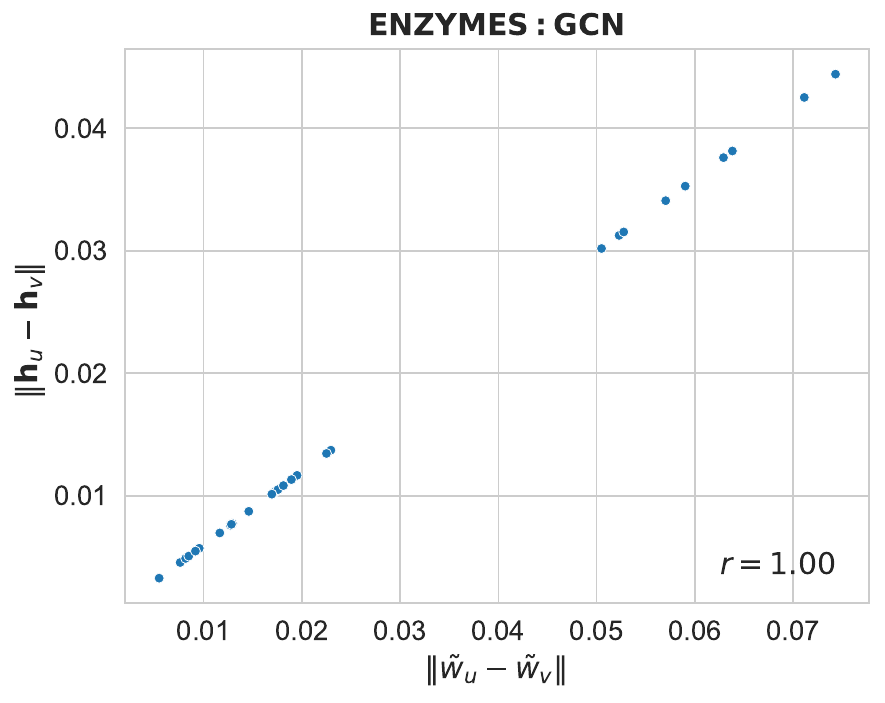}
    \end{subfigure}
    \caption{Euclidean distances of the representations generated at the third layer of the different models vs. Euclidean distances of the number of walks (or sum of normalized walks) of length $3$ starting from the different nodes.
    Nodes $v$ and $u$ correspond to the same node in the original and the perturbed graph, respectively.
    Each perturbed graph has emerged by removing one edge from node $v$'s neighborhood.}
    \label{fig:corr_gnns_edge}
\end{figure*}

\end{document}